\documentclass{article}

\usepackage{microtype}
\usepackage{graphicx}
\usepackage{subfigure}
\usepackage{booktabs} 

\usepackage{hyperref}

\usepackage{amsfonts} 
\usepackage{amsmath}
\usepackage{amssymb}
\usepackage{amsthm}
\usepackage{cleveref}

\crefformat{footnote}{#2\footnotemark[#1]#3}

\usepackage{bbm}

\usepackage[accepted]{icml2020}

\usepackage{cleveref}
\crefformat{footnote}{#2\footnotemark[#1]#3}

\renewcommand{\vec}[1]{\mathbf{#1}}	

\newcommand{\argmax}{\operatornamewithlimits{argmax}}

\newcommand{\reldiam} {\mathrm{d}} 
\newcommand{\Xset} {\mathcal{X}} 
\newcommand{\Zset} {\mathcal{Z}}

\newtheorem{definition}{Definition}
\newtheorem{theorem}{Theorem}
\newtheorem{lemma}{Lemma}

\theoremstyle{remark}
\newtheorem{remark}{Remark}

\icmltitlerunning{Private Outsourced Bayesian Optimization}

\def\myproof{1}

\begin{document}

\twocolumn[
\icmltitle{Private Outsourced Bayesian Optimization}




\begin{icmlauthorlist}
\icmlauthor{Dmitrii Kharkovskii}{nus}
\icmlauthor{Zhongxiang Dai}{nus}
\icmlauthor{Bryan Kian Hsiang Low}{nus}

\end{icmlauthorlist}

\icmlaffiliation{nus}{Department of Computer Science, National University of Singapore, Republic of Singapore}
\icmlcorrespondingauthor{Bryan Kian Hsiang Low}{lowkh@comp.nus.edu.sg}

\icmlkeywords{Machine Learning, ICML}

\vskip 0.3in
]



\printAffiliationsAndNotice{} 

\begin{abstract} 
This paper presents the \emph{private-outsourced-Gaussian process-upper confidence bound} (PO-GP-UCB) algorithm, which is
the first algorithm for privacy-preserving \emph{Bayesian optimization} (BO) 
in the \emph{outsourced} setting with a provable performance guarantee. 
We consider the outsourced setting where the entity holding the dataset and the entity performing BO are represented by different parties, and the dataset cannot be released non-privately.
For example, a hospital holds a dataset of sensitive medical records and outsources the BO task on this dataset to an industrial AI company.
The key idea of our approach is to make the BO performance of our algorithm similar to that of non-private GP-UCB run using the original dataset,
which is achieved by using a random projection-based transformation that preserves both privacy and the pairwise distances between inputs.
Our main theoretical contribution is to show that a regret bound similar to that of the standard GP-UCB algorithm can be established for our PO-GP-UCB algorithm.
We empirically evaluate the performance of our PO-GP-UCB algorithm with synthetic and real-world datasets.
\end{abstract}

\section{Introduction} 
\label{sec:intro}
\emph{Bayesian optimization} (BO)  has become an increasingly popular method for optimizing highly complex black-box functions mainly due to its impressive sample efficiency.
Such optimization problems appear frequently in various applications such as automated machine learning (ML), robotics, sensor networks, among others \cite{shahriari16}. 
However, despite its popularity, the classical setting of BO does not account for privacy issues, 
which arise due to the widespread use of ML models in applications dealing with sensitive datasets such as health care~\cite{yu13}, insurance~\cite{chong05} and fraud detection~\cite{ngai2011}. 
A natural solution is to apply the cryptographic framework of \emph{differential privacy} (DP)~\cite{dwork06}, which has become the state-of-the-art technique for private data release
and has been widely adopted in ML~\cite{sarwate2013}.

To this end, a recent work~\cite{kusner15} proposed a DP variant of the \emph{Gaussian process-upper confidence bound} (GP-UCB) algorithm, 
which is a well-known BO algorithm with theoretical performance guarantee~\cite{srinivas10}.
\citet{kusner15} consider the common BO task of hyperparameter tuning for ML models and 
introduce methods for privatizing the outputs of the GP-UCB algorithm (the optimal input hyperparameter setting found by the algorithm and the corresponding output measurement) by 
releasing them using standard DP mechanisms.
However, in many scenarios, BO is performed in the \emph{outsourced} setting, in which the entity holding the sensitive dataset (referred to as the \emph{curator} hereafter) 
and the entity performing BO (referred to as the \emph{modeler} hereafter) are represented by different parties with potentially conflicting interests. 
In recent years, such modelers (i.e., commercial companies) providing general-purpose optimization as a service have become increasingly prevalent, such as SigOpt 
which uses BO as a commercial service for black-box global optimization by providing query access to the users~\cite{dewancker2016evaluation}, and Google Cloud AutoML 
which offers the optimization of the architectures of neural networks as a cloud service.
Unfortunately, the approach of \citet{kusner15} requires the modeler and the curator to be represented by the same entity and 
therefore both parties must have complete access to the sensitive dataset and full understanding of the BO algorithm, thus rendering it inapplicable in the outsourced setting.
Some examples of  such 
settings are given below: 

\noindent (a)  
A hospital is trying to find out which patients are likely to be readmitted soon based on the result of an expensive medical test~\cite{yu13}. 
Due to cost and time constraints, the hospital (curator) is only able to perform the test for a limited number of patients, and 
thus outsources the task of selecting candidate patients for testing to an industrial AI company (modeler). 
In this case, the inputs to BO are medical records of individual patients and the function to maximize (the output measurement) is the outcome of the medical test for different patients,
which is used to assess the possibility of readmission.
The hospital is unwilling to release the  medical records, while the AI company does not want to share the details of their proprietary algorithm. 

\noindent (b) A bank aims 
to identify the loan applicants with the highest return on investment
and outsources the task to a financial AI consultancy.
In this case, each input to BO is the data of a single loan applicant and the output measurement to be maximized is the return on investment for different applicants.
The bank (curator) is unable to disclose the raw data of the loan applicants due to privacy and security concerns, 
whereas the AI consultancy (modeler) is unwilling to share the implementation of their selection strategy.

\noindent (c) A real estate agency 
attempts to locate the cheapest private properties in an urban city. 
Since evaluating every property requires sending an agent to the corresponding location, the agency (curator) outsources the task of selecting candidate properties for evaluation
to an AI consultancy (modeler) to save resources. 
Each input to BO is a set of features representing a single property and the function to minimize (the output measurement) is the evaluated property price. 
The agency is unable to disclose the particulars of their customers due to legal implications, while the AI consultancy refuses to share their decision-making algorithm.

In all of these scenarios, the curator is unable to release the original dataset
due to privacy concerns, and therefore has to provide a transformed privatized dataset to the modeler. 
Then, the modeler can perform BO (specifically, the GP-UCB algorithm) on the transformed dataset (the detailed problem setting is described in Section~\ref{sec:prelim} and illustrated in Fig.~\ref{fig:setting}).
A natural choice for the privacy-preserving transformation is to use standard DP methods such as the Laplace or Gaussian mechanisms~\cite{dwork14}.
However, the theoretically guaranteed convergence of the GP-UCB algorithm~\cite{srinivas10} is only valid if it is run using the original dataset.
Therefore, as a result of the privacy-preserving transformation required in the outsourced setting, 
it is unclear whether the theoretical guarantee of GP-UCB can be preserved and thus whether reliable performance can be delivered.
This poses an interesting research question: \emph{How do we design a privacy-preserving algorithm for outsourced BO with a provable performance guarantee?}

To address this challenge, we propose the \emph{private-outsourced-Gaussian process-upper confidence bound} (PO-GP-UCB) algorithm (Section~\ref{sec:main}), 
which is the first algorithm for BO with DP in the outsourced setting with a provable performance guarantee\footnote{While the setting of the recent work of \citet{nguyen2018} resembles ours, 
the authors use a self-proposed notion of privacy instead of the widely recognized DP and protect the privacy of only the output measurements. 
}. 
The key idea of our approach is to make the GP predictions and hence the BO performance of our algorithm similar to those of non-private GP-UCB run using the original dataset.
To achieve this, instead of standard DP methods, we use a privacy-preserving transformation based on random projection~\cite{johnson84}, 
which approximately preserves the pairwise distances between inputs. 
We show that preserving the pairwise distances between inputs leads to preservation of the GP predictions and therefore the BO performance in the outsourced setting
(compared with the standard setting of running non-private GP-UCB on the original dataset).
Our main theoretical contribution is to show that a regret bound similar to that of the standard GP-UCB algorithm can be established for our PO-GP-UCB algorithm.
We empirically evaluate the performance of our PO-GP-UCB algorithm with synthetic and real-world datasets (Section~\ref{sec:expt}).

\section{Notations and Preliminaries}

\begin{figure}
	\includegraphics[width=0.48 \textwidth]{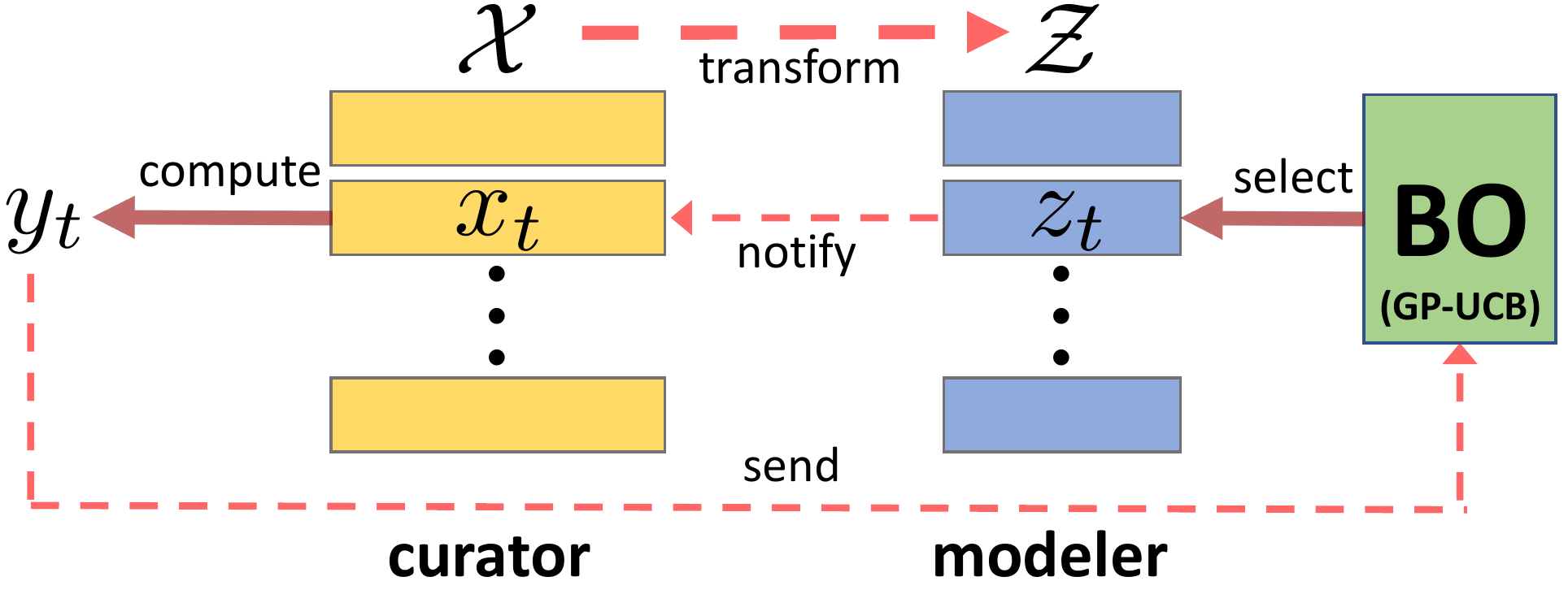} \vspace{-5.5mm}
	\caption{Visual illustration of the problem setting.}
	\label{fig:setting} \vspace{-2.4mm}
\end{figure}
\label{sec:prelim} 
\noindent
{\bf Problem Setting.} Privacy-preserving BO in the outsourced setting involves two parties: the \textit{curator} who holds the sensitive dataset (e.g., a list of medical records), 
and the \textit{modeler} who performs the outsourced BO on the transformed dataset provided by the curator (see Fig.~\ref{fig:setting} for a visual illustration of this setting).
The curator holds the original dataset represented as a set $\Xset  \subset \mathbb{R}^{d}$ formed by $n$ $d$-dimensional inputs.
The curator and the modeler intend to maximize an unknown expensive-to-evaluate objective function $f$ defined over $\Xset$.
At the beginning, the curator performs a privacy-preserving transformation of the original dataset $\Xset$ to obtain a transformed dataset $\Zset \subset \mathbb{R}^{r}$ formed by $n$ $r$-dimensional inputs.  
As a result, every original input $x \in \Xset$ has an image, which is the corresponding transformed input $z \in \Zset$.
Then, the curator releases the transformed dataset $\Zset$ to the modeler, who can subsequently start to run the BO algorithm on $\Zset$.
In each iteration $t=1, \ldots, T$, the modeler selects a transformed input $z_t \in \Zset$
to query and notifies the curator about the choice of $z_t$. 
Next, the curator identifies $x_t$ which is the preimage of $z_t$ under the privacy-preserving transformation\footnote{We assume that $\Xset$ and $\Zset$ describe the entire optimization domain, i.e., every $z_t \in \Zset$ has a preimage  $x_t \in \Xset$.},
and then computes $f(x_t)$ to yield a noisy output measurement: $y_t \triangleq  f(x_t) + \epsilon_{GP}$, in which $\epsilon _{GP}\sim \mathcal{N}(0,\sigma^2_n)$ is a zero-mean Gaussian noise with noise variance $\sigma^2_n$.
We assume that $y_t$ is unknown to the curator in advance and is computed only when requested by the modeler, which is reasonable in all motivating scenarios in Section~\ref{sec:intro}. 
The curator then sends $y_t$ to the modeler for performing the next iteration of BO. 
We have assumed that in contrast to the input $x_t$, the noisy output measurement $y_t$ does not contain sensitive information and can thus be non-privately released. 
This assumption is reasonable in our setting, e.g., if $y_t$ represents the outcome of a medical test, revealing $y_t$ does not unveil the identity of the patient.
We leave the extension of privately releasing $y_t$ for future work (see Section~\ref{sec:conclusion}).

\noindent {\bf Differential Privacy (DP).} Differential privacy~\cite{dwork06} has become the state-of-the-art technique for private data release. 
DP is a cryptographic framework which provides rigorous mathematical guarantees on privacy, typically by adding some random noise during the execution of the data release algorithm. 
DP has been widely adopted by the ML community (see the work of~\citet{sarwate2013} for a detailed survey).
Intuitively, DP promises that changing a single input of the dataset imposes only a small change in the output of the data release algorithm, hence the output does not depend significantly on any individual input. 
As a result, an attacker is not able to tell if an input is changed in the dataset just by looking at the output of the data release algorithm.
To define DP, we first need to introduce the notion of \textit{neighboring} datasets.
Following the prior works on DP~\cite{blocki2012,hardt12}, we define two neighboring datasets as those differing only in a single row (i.e., a single input) with the norm of the difference bounded by $1$: 
\begin{definition}
	\label{def:neighb}
	Let $\Xset , \Xset ' \in \mathbb{R}^{n \times d}$ denote two datasets viewed as matrices\footnote{\label{footnote-matrix} We slightly abuse the notation and view the dataset $\Xset$ ($\Zset$) as an $n \times d$ ($n \times r$) matrix where each of the $n$ rows corresponds to an original (transformed) input.} with $d$-dimensional inputs $\{x^{}_{(i)} \}_{i=1}^n$ and $\{x'_{(i)} \}_{i=1}^n$ as rows respectively. 
	We call  datasets $\Xset $ and $\Xset '$  neighboring  if there exists an index $i^* \in 1, \ldots, n$ such that $\Vert x^{}_{(i^*)} - x'_{(i^*)} \rVert \le 1$, and 
	$\lVert x^{}_{(j)} - x'_{(j)} \rVert = 0$ for any index $j \in 1, \ldots, n$, $j \neq i^*$.
\end{definition} 
A randomized algorithm is differentially private if, for any two neighboring datasets,
the distributions of the outputs of the algorithm calculated on these datasets are similar:
\begin{definition}
	\label{def:dp}
	A randomized algorithm $\mathcal{M}$ is $(\epsilon, \delta)$-differentially private for $\epsilon > 0$ and $\delta \in (0,1)$ if, for all  $O \subset \text{Range}(\mathcal{M})$ (where $\text{Range}(\mathcal{M})$ is the range of the outputs of the randomized algorithm $\mathcal{M}$) and for all neighboring datasets $\Xset $ and $\Xset '$, we have that
	\begin{equation*}
	P(\mathcal{M}(\Xset ) \in O) \le \exp (\epsilon) \cdot P(\mathcal{M}(\Xset ') \in O) + \delta.
	\end{equation*} 
\end{definition} 
Note that the definition above is symmetric in terms of $\Xset $ and $\Xset '$. 
The DP parameters $\epsilon, \delta$  control the \emph{privacy-utility trade-off}: 
The smaller they are, the tighter the privacy guarantee is, at the expense of lower accuracy due to increased amount of noise required to satisfy DP.  The state-of-the-art works on the application of DP in ML~\cite{abadi16,foulds2016,papernot2016} use the values of $\epsilon$ in the single-digit range, while the value of $\delta$ is usually set to be smaller than $1/n$~\cite{dwork14}.
Refer to the work of~\citet{dwork14} for more details about DP.

\noindent {\bf Bayesian Optimization (BO).} We consider the problem of sequentially maximizing an unknown objective function $f: \Xset  \rightarrow \mathbb{R}$, in which $\Xset  \subset \mathbb{R}^d$ denotes a domain of $d$-dimensional inputs. 
We consider the domain to be discrete for simplicity.
In the classical setting of BO, in each iteration $t=1,\ldots,T$, an unobserved input 
$x_t \in \Xset $  is selected to query by maximizing an \emph{acquisition function} (AF), 
yielding a noisy output measurement $y_t \triangleq f(x_t) + \epsilon_{GP}$,
in which $\epsilon_{GP} \sim \mathcal{N}(0,\sigma^2_n)$ is a zero-mean Gaussian noise with noise variance $\sigma^2_n$.
The AF should be designed to allow us to approach the global maximum $f(x^*)$ rapidly, in which $x^* \triangleq \argmax_{x \in \Xset } f(x)$. 
This can be achieved by minimizing a standard BO objective such as \emph{regret}.
The notion of regret intuitively refers to a loss in reward resulting from not knowing $x^*$ beforehand. 
Formally, the \emph{instantaneous regret} incurred in iteration $t$ is defined as $r_{t} \triangleq f(x^*) - f(x_t)$. 
\emph{Cumulative regret} is defined as the sum of all instantaneous regrets, i.e., $R_T \triangleq \sum_{t=1}^T  r_{t}$, 
and \emph{simple regret} is defined as the minimum among all instantaneous regrets, i.e., 
$S_T \triangleq \min_{t=1,\ldots, T} r_t$.
It is often desirable for a BO algorithm to be asymptotically \emph{no-regret},
i.e., $\lim_{T \rightarrow \infty}S_T \leq \lim_{T \rightarrow \infty} R_T/T = 0$,
which implies that convergence to the global maximum is guaranteed.

\noindent {\bf Gaussian Process (GP).}  
In order to facilitate the design of the AF to minimize the regret, we model our belief of the unknown objective function $f$  using a GP.
Let $ f(x) _{x \in \Xset}$ denote a GP, that is, every finite subset of $ f(x) _{x\in \Xset}$ follows a multivariate Gaussian distribution~\cite{gpml}.
Then, the GP is fully specified by its \emph{prior} mean $\mu_x \triangleq \mathbb{E}[f(x)]$ and covariance $k_{xx'} \triangleq \mbox{cov}[f(x), f(x')]$ for all $x, x'\in \Xset$. 
We assume that $k_{xx'}$ is defined by the commonly-used isotropic\footnote{\label{footnote-isotropic} The non-isotropic squared exponential covariance function for  $x, x'\in \Xset$ is defined as $k_{xx'} \triangleq \sigma_{y}^2\exp \{(x-x')^{\top}\Gamma^{-2}(x-x') \}$, in which $\Gamma$ is a diagonal matrix with length-scale components $[l_1, \ldots, l_d]$. It can be easily transformed to an isotropic one by preprocessing the inputs, i.e., dividing each dimension of inputs $x, x'$ by the respective length-scale component.} 
squared exponential covariance function $k_{xx'} \triangleq \sigma_{y}^2\exp \{-0.5 \lVert x-x' \rVert^2 / l^2 \}$, in which $\sigma_{y}^2$ is the signal variance controlling the intensity of  output measurements and $l$ is the length-scale controlling the correlation or ``similarity'' between output measurements. 
Furthermore, without loss of generality, we assume $\mu_x=0$ and $k_{xx'} \le 1$ for all $x, x'\in \Xset$. 
Given a column vector $\vec{y}_{t} \triangleq  [y_i]_{1, \ldots, t}^\top$ of noisy output measurements for some set $\Xset_{t}\triangleq  \{ x_1, \ldots, x_{t}  \}$ of inputs after $t$ iterations, 
the \emph{posterior} predictive distribution of $f(x)$ at any input $x$ is a Gaussian distribution with the following posterior mean and variance: \vspace{-5.0mm}
\begin{equation} 
\label{eq:posterior}
\begin{array}{l}
\displaystyle \mu_{t+1 }(x) \triangleq   K_{x  \Xset_{t}} (K_{\Xset_{t} \Xset_{t}} + \sigma^2_n  I)^{-1} \vec{y}_{t} \\ 
\displaystyle \sigma^2_{t+1}(x) \triangleq k_{xx} -  K_{ x  \Xset_{t}}
( K_{\Xset_{t}\Xset_{t}} + \sigma^2_n I)^{-1} K_{\Xset_{t}  x}, 
\end{array}
\end{equation}
in which $K_{x \Xset_{t}}\triangleq(k_{xx'})_{x'\in \Xset_{t}}$ is a row vector, $K_{\Xset_{t} x} \triangleq K_{x \Xset_{t}}^\top$, and $K_{\Xset_{t} \Xset_{t}}\triangleq(k_{x'x''})_{x', x''\in \Xset_{t}}$.

Under the privacy-preserving transformation (Fig.~\ref{fig:setting}), denote the image of the set $\Xset_{t}$ as $\Zset_{t}\triangleq  \{ z_1, \ldots, z_{t}  \}$, and 
the images of the original inputs $x$ and $x'$ as $z$ and $z'$ respectively.
Then, the covariance function $k_{zz'}$ can be defined similarly as $k_{xx'}$, and 
thus we can define an analogue of the predictive distribution~\eqref{eq:posterior} for $z$ and $\Zset_t$ (instead of $x$ and $\Xset_t$), which we denote as $\tilde{\mu}_{t+1}(z)$ and $\tilde{\sigma}^2_{t+1 }(z)$. 
We assume that the function $f$ is sampled from a GP defined over the original domain $\Xset$ with the covariance function $k_{xx'}$ and with known hyperparameters ($\sigma_{y}^2$ and $l$), 
and we use the same hyperparameters for the covariance function $k_{zz'}$.

\noindent {\bf The GP-UCB Algorithm.}
The AF adopted by the GP-UCB algorithm~\cite{srinivas10} is 
the \emph{upper confidence bound} (UCB) of $f$ induced by the posterior GP model. 
In each iteration $t$, an input $x_t \in \Xset $ is selected to query by trading off between 
\emph{(a)} sampling close to an expected maximum (i.e., with large posterior mean $\mu_{t }(x_t)$) given the current GP belief (i.e., exploitation) vs. 
\emph{(b)} sampling an input with high predictive uncertainty (i.e., with large posterior standard deviation $\sigma_{t}(x_t)$) to
improve the GP belief of $f$ over $\Xset$ (i.e., exploration). 
Specifically, $x_t \triangleq {\argmax}_{x \in \Xset }\ \mu_{t }(x)  + \beta^{\text{ \tiny{$1/2$} }}_t \sigma_{t}(x) $, in which the parameter $\beta_t > 0$ is set to trade off between exploitation vs. exploration. 
A remarkable property of the GP-UCB algorithm shown by the work of
~\citet{srinivas10} is that it achieves \emph{no regret} asymptotically if the parameters $\beta_t > 0$ are chosen properly.

A recent work by~\citet{kusner15} proposed a DP variant of GP-UCB for hyperparameter tuning of ML models. 
However, as mentioned in Section~\ref{sec:intro}, this approach implies that the modeler and curator are represented by the same entity
and thus both parties have full access to the sensitive dataset and detailed knowledge of the BO algorithm.
In our outsourced setting, in contrast,  the modeler only has access to the transformed privatized dataset, while the curator is unaware of the details of the BO algorithm, as described in our motivating scenarios (Section~\ref{sec:intro}). 

\section{Outsourced Bayesian Optimization}
\label{sec:main}
In our PO-GP-UCB algorithm,
the curator needs to perform a privacy-preserving transformation of the original dataset  $\Xset \subset \mathbb{R}^{d}$ and release the transformed dataset $\Zset \subset \mathbb{R}^{r}$ to the modeler. 
Subsequently, the modeler runs BO (i.e., GP-UCB) using $\Zset$.
When performing the transformation, the goal of the curator is two-fold: Firstly, the transformation has to be differentially private with given DP parameters $\epsilon, \delta$ (Definition~\ref{def:dp});
secondly, the transformation should allow the modeler to obtain good BO performance on the transformed dataset (in a sense to be formalized later in this section).

\subsection{Transformation via Random Projection}
Good BO performance by the modeler (i.e., the second goal of the curator) can be achieved by making the GP predictions~\eqref{eq:posterior} (on which the performance of the BO algorithm depends)
using the transformed dataset $\Zset$ close to those using the original dataset $\Xset$.
To this end, we ensure that the distances between all pairs of inputs are approximately preserved after the transformation.
This is motivated by the fact that the GP predictions~\eqref{eq:posterior} (hence the BO performance) depend on the inputs only through the value of covariance, 
which, in the case of isotropic covariance functions\cref{footnote-isotropic}, only depends on the pairwise distances between inputs.
Consequently, by preserving the pairwise distances between inputs, the performance of the BO (GP-UCB) algorithm run by the modeler on $\Zset$ is made similar to that of the non-private GP-UCB algorithm run on the original dataset $\Xset$, 
for which theoretical convergence guarantee has been shown~\cite{srinivas10}.
As a result, the BO performance in the outsourced setting can be theoretically guaranteed (Section~\ref{sec:subsec_modeler_part}) and thus practically assured.

Therefore, to achieve both goals of the curator, we need to address the question as to \emph{what transformation preserves both the pairwise distances between inputs and DP}.
A natural approach is to add noise directly to the matrix of pairwise distances between the original inputs from $\Xset$ using standard DP methods such as the Laplace or Gaussian mechanisms~\cite{dwork14}. 
However, the resulting noisy distance matrix is not guaranteed to produce an invertible covariance matrix $K_{\Xset_{t}\Xset_{t}} + \sigma^2_n I$, which is a requirement for the GP predictions~\eqref{eq:posterior}. 
Instead, we perform the transformation through a technique based on random projection, which satisfies both goals of the curator. 
Firstly, random projection through random samples from standard normal distribution has been shown to preserve DP~\cite{blocki2012}.
Secondly, as a result of the Johnson-Lindenstrauss lemma~\cite{johnson84},
random projection is also able to approximately preserve the pairwise distances between inputs, as shown in the following lemma:

\begin{lemma}
	\label{lemma:jl}
	Let $\nu \in (0, 1/2)$, $\mu \in (0, 1)$, $d \in\mathbb{N}$ and a set  $\Xset \subset  \mathbb{R}^d$ of $n$ row vectors be given.
	Let $r\in\mathbb{N}$ and $M$ be a $d \times r$ matrix whose entries are i.i.d. samples from $\mathcal{N}(0, 1)$. If $r \ge 8 \log (n^2 / \mu) / \nu^2$, the probability of
	\begin{equation*}
	(1 - \nu) \lVert x - x' \rVert^2  \leq r^{-1} \lVert x M - x' M \rVert^2  \le (1 + \nu) \lVert x - x' \rVert^2  
	\end{equation*}
	for all  $x, x' \in \Xset$ is at least $1 - \mu$.
\end{lemma}	

\begin{remark}
	\label{rem:jl}
	$r$ controls the dimension of the random projection, while $\nu$ and $\mu$ control the accuracy. 
	Lemma~\ref{lemma:jl} corroborates the intuition that a smaller value of $r$ leads to larger values of $\nu$ and $\mu$, i.e., lower random projection accuracy.
\end{remark}

The proof (Appendix~\ref{app:rem1}) consists of a union bound applied to the Johnson-Lindenstrauss lemma~\cite{johnson84}, 
which is a result from geometry stating that a set of points in a high-dimensional space can be embedded into a lower-dimensional space such that the pairwise distances between the points are nearly preserved.

 \vspace{-1.4mm}
\subsection{The Curator Part}
\vspace{-0.9mm}
The curator part (Algorithm~\ref{alg:jl-dp}) of our PO-GP-UCB algorithm takes as input the original dataset $\Xset$ viewed as an $n \times d$ matrix\cref{footnote-matrix}, 
the DP parameters $\epsilon$, $\delta$ (Definition~\ref{def:dp}) and the random projection parameter $r$ (Lemma~\ref{lemma:jl})\footnote{\label{footnote-r-mu} 
	Note that in Theorem~\ref{th:main}, the parameter $r$ is calculated based on specific values of the parameters $\mu$ and $\nu$ (Lemma~\ref{lemma:jl}) 
	in order to achieve the performance guarantee. However, in practice, $\mu$ and $\nu$ are not required to specify the value of $r$ for Algorithm~\ref{alg:jl-dp}.}.
To begin with, the curator
subtracts the mean from each column of $\Xset$ (line $2$), and then picks a matrix $M$ of samples from standard normal distribution $\mathcal{N}(0, 1)$ to perform random projection (line $3$). 
Next, if the smallest singular value $\sigma_{min}(\Xset)$ of the centered dataset $\Xset$ is not less than a threshold $\omega$ (calculated in line $5$), 
the curator outputs  the random projection $\Zset \triangleq r^{-1/2} \Xset  M$ of the centered dataset $\Xset$ (line $7$). 
Otherwise, the curator increases the singular values of the centered dataset  $\Xset $ (line $9$) to obtain a new dataset  $\tilde{\Xset }$ and 
outputs the random projection $\Zset \triangleq r^{-1/2} \tilde{\Xset}  M$  of the new dataset  $\tilde{\Xset}$ (line $10$). 
Lastly, the curator releases $\Zset$ to the modeler (line $11$).

\begin{algorithm}		
	\caption{PO-GP-UCB (The curator part)}
	\label{alg:jl-dp}
	\begin{algorithmic}[1]
		\STATE \textbf{Input:} 	$\Xset$, $\epsilon, \delta$, $r$
		\STATE 
		$\Xset  \leftarrow \Xset  -  \mathbf{1} \mathbf{1}^\top \Xset  / n$ where $\mathbf{1}$ is a $n \times 1$ vector of $1$'s
		\STATE Pick a $d \times r$ matrix $M$ of i.i.d. samples from $\mathcal{N}(0, 1)$
		\STATE Compute the SVD of $\Xset  = U\Sigma V^\top$
		\STATE $\omega \gets 16 \sqrt{r \log (2 / \delta)} \epsilon^{-1} \log (16 r / \delta)$
		\IF {$\sigma_{min}(\Xset) \ge \omega$} 
		\STATE \bf{return} $\Zset \gets  r^{-1/2} \Xset  M$ 
		\ELSE \STATE $\tilde{\Xset } \leftarrow U \sqrt{\Sigma^2  + \omega^2 I_{n \times d}} V^\top$ where $\Sigma^2 $ ($I_{n\times d}$) is an $n\times d$ matrix whose main diagonal has squared singular values of $\Xset$ (ones) in each coordinate and all other coordinates are $0$
		\STATE \bf{return} $\Zset \gets r^{-1/2} \tilde{\Xset } M$ 
		\ENDIF  
		\STATE Release dataset  $\Zset$ to the modeler
	\end{algorithmic}
\end{algorithm}

The fact that Algorithm~\ref{alg:jl-dp} both preserves DP and approximately preserves the pairwise distances between inputs is stated in 
Theorems~\ref{th:jl-dp} and~\ref{lemma:dist} below.
\begin{theorem}
	\label{th:jl-dp}
	Algorithm~\ref{alg:jl-dp} preserves $(\epsilon, \delta)$-DP.	
\end{theorem}

In the proof of Theorem~\ref{th:jl-dp} (Appendix~\ref{app:th-jl-dp}), all singular values of the dataset $\Xset$ are required to be not less than $\omega$ (calculated in line $5$). 
This explains the necessity of line $9$, where we increase the singular values of the dataset $\Xset$ if $\sigma_{min}(\Xset) <  \omega$, to ensure that this requirement is satisfied.

\begin{theorem}
	\label{lemma:dist}
	Let  a dataset  $\Xset  \subset \mathbb{R}^{d}$  be given. Let $\nu \in (0, 1/2)$, $\mu \in (0, 1)$ be given. Let $r\in\mathbb{N}$, such that $r \ge 8 \log (n^2 / \mu) / \nu^2$.
	Then, the probability of
	\begin{equation*}	
	(1 - \nu) \lVert x - x' \rVert^2  \leq \lVert z - z' \rVert^2  \le (1 + \nu) C' \lVert x - x' \rVert^2
	\end{equation*}
	for all $x, x' \in \Xset$ and their images $z, z' \in \Zset$ is at least $1 - \mu$, in which
	$C' \triangleq 1 + \mathbbm{1}_{\sigma_{min}(\Xset) < \omega} \omega^2 /\sigma_{min}^2(\Xset)$.
\end{theorem}

The proof (Appendix~\ref{sef:lemma_dist_proof}) consists of  bounding the change in distances between inputs due to the increase of the singular values of the dataset $\Xset$ (line $9$ of Algorithm~\ref{alg:jl-dp}) and 
applying Lemma~\ref{lemma:jl}. It can be observed from Theorem~\ref{lemma:dist} that when $\sigma_{min}(\Xset) \ge \omega$, $C' = 1$, hence Algorithm~\ref{alg:jl-dp} approximately preserves the pairwise distances between inputs. 

There are several important differences between our Algorithm~\ref{alg:jl-dp} and the work of~\citet{blocki2012}. Firstly, Algorithm~$3$ of~\citet{blocki2012} releases a DP estimate 
of the dataset covariance matrix,
 while our Algorithm~\ref{alg:jl-dp} outputs a DP transformation of the original dataset. 
Secondly, Algorithm~$3$ of~\citet{blocki2012} does not have the ``if/else'' condition (line $6$ of Algorithm~\ref{alg:jl-dp}) and 
always increases the singular values as in line $9$ of Algorithm~\ref{alg:jl-dp}. 
In our case, however, if the singular values are increased due to the condition $\sigma_{min}(\Xset) < \omega$ (i.e., the ``else'' clause, line $8$ of Algorithm~\ref{alg:jl-dp}),
the pairwise input distances of the dataset $\Xset$ are no longer approximately preserved in $\Zset$ (Theorem~\ref{lemma:dist}), 
which results in a slightly different regret bound (see Theorem~\ref{th:main} and Remark~\ref{rem:2} below).
This requires us to introduce  the ``if/else'' condition in Algorithm~\ref{alg:jl-dp}.
We discuss these changes in greater detail in Appendix~\ref{app:th-jl-dp}.

\subsection{The Modeler Part}

\label{sec:subsec_modeler_part}
The modeler part of our PO-GP-UCB algorithm (Algorithm~\ref{alg:jl-modeler})  takes as input the transformed dataset  $\Zset \subset \mathbb{R}^r$ received from the curator as well as the GP-UCB parameter $\delta'$,
and runs the GP-UCB algorithm for $T$ iterations on $\Zset$. 
In each iteration $t$, the modeler selects the candidate transformed input $z_t$ by maximizing the GP-UCB AF (line $4$), 
and queries the curator for the corresponding noisy output measurement $y_t$ (line $5$). 
To perform such a query, the modeler can send the index (row) $i_t$ of the selected transformed input $z_t$ in the dataset $\Zset$ viewed as a matrix\cref{footnote-matrix}
to the curator. The curator can then find the preimage $x_t$ of $z_t$ by looking into the same row $i_t$ of the dataset $\Xset$ viewed as a matrix\cref{footnote-matrix}.
After identifying $x_t$, the curator can compute $f(x_t)$ to yield a noisy output measurement $y_t\triangleq f(x_t) + \epsilon_{GP}$ and send it to the modeler.
The modeler then updates the GP posterior belief (line $6$) and proceeds to the next iteration $t+1$.

In our theoretical analysis, we make the 
assumption of the \emph{diagonal dominance} property of the covariance matrices, 
which was used by previous works on GP with DP~\cite{smith18} and active learning~\cite{hoang14}:
\begin{definition}
	\label{definition-diag}
	Let a dataset $\Xset  \subset \mathbb{R}^{d}$ 
	and a set $\Xset_0 \subseteq \Xset$ be given. The covariance matrix $K_{\Xset_0 \Xset_0}$ is said to be diagonally dominant if for any $x \in \Xset_0$
	\begin{equation*}
	k_{x x} \geq \big( \sqrt{| \Xset_0 |  -1} + 1 \big) \sum \nolimits_{x' \in \Xset_0 \setminus  x } k_{x x'}.
	\end{equation*}
\end{definition}

Note that this assumption is adopted mainly for the theoretical analysis, and is thus not strictly required in order for our algorithm to deliver competitive practical performance (Section~\ref{sec:expt}).
Theorem~\ref{th:main} below presents the theoretical guarantee on the BO performance of our PO-GP-UCB algorithm run by the modeler (Algorithm~\ref{alg:jl-modeler}).

\begin{algorithm}		
	\caption{PO-GP-UCB (The modeler part)}
	\label{alg:jl-modeler}
	\begin{algorithmic}[1]
		\STATE \textbf{Input:} $\Zset$,  $\delta'$, $T$ 
		\FOR{$t=1, \ldots, T$}
		\STATE Set $\beta_t \gets 2\log( n t^2 \pi^2 / 6 \delta')$
		\STATE  	$z_t \gets \argmax_{z \in \Zset}\tilde{\mu}_{t }(z) + \beta_t^{1/2} \tilde{\sigma}_{t}(z)$\\
		\STATE  Query the curator  for $y_t$
		\STATE Update GP posterior belief: $\tilde{\mu}_{t + 1 }(z) $ and $ \tilde{\sigma}_{t+1}(z)$
		\ENDFOR
	\end{algorithmic}
\end{algorithm}

\begin{theorem}
	\label{th:main}
	Let $\varepsilon_{ucb} > 0$, $\delta_{ucb} \in (0,1)$, $T \in \mathbb{N}$, DP parameters $\epsilon$ and $\delta$, and a dataset  $\Xset  \subset \mathbb{R}^{d}$ be given.  
	Let  $\reldiam \triangleq \text{diam}(\Xset) / l$ where $\text{diam}(\Xset)$ is the diameter of $\Xset$ and $l$ is the GP length-scale. 
	Suppose for all $t = 1, \ldots, T$, $| y_t | \le L$ and $K_{\Xset_{t-1} \Xset_{t-1}}$ is diagonally dominant. 
	Suppose
	$r \ge 8 \log (n^2 / \mu) / \nu^2$ (Algorithm~\ref{alg:jl-dp}) where $\mu \triangleq \delta_{ucb} / 2$ and $\nu\triangleq \min(\varepsilon_{ucb}/ (2 \sqrt{3} \reldiam^2 L), 2 / \reldiam^2, 1/2)$, and
	$\delta' \triangleq \delta_{ucb}  / 2$ (Algorithm~\ref{alg:jl-modeler}). 
	If $\sigma_{min}(\Xset) \ge \omega$, 
	then the simple regret $S_T$ incurred by Algorithm~\ref{alg:jl-modeler} run by the modeler satisfies
	\begin{equation*}
	\begin{array}{rl}
	S_T  \le &  \big( \varepsilon^2_{ucb} 
	+ 24 (C_2  + C_1 \beta^{1/2}_T)^2\log T / T \\
	& + 24 / \log (1 + \sigma_n^{-2}) \cdot   \beta_T \gamma_T	/ T \big)^{1/2}
	\end{array}
	\end{equation*}
	with probability at least $1-\delta_{ucb}$, in which $\gamma_T$ is the
	maximum information gain\footnote{\citet{srinivas10} has shown that $\gamma_T=\mathcal{O}((\log T)^{d+1})$ for the squared exponential kernel.} on the function $f$ from any set of noisy output measurements of size $T$, 
	$C_1 \triangleq \mathcal{O}\big(\sigma_y \sqrt{\sigma_{y}^2  + \sigma^2_n} ( \sigma^2_y  / \sigma_n^2 + 1  )\big)$ and $C_2\triangleq\mathcal{O}(\sigma^2_y / \sigma^2_n \cdot L)$.
\end{theorem}
The key idea of the proof (Appendix~\ref{sec:main_proof}) is to ensure that every value of the GP-UCB AF 
computed on the transformed dataset  $\Zset$ is close to the value of the corresponding GP-UCB AF computed
on the original dataset  $\Xset$. 
Consequently, the regret of the PO-GP-UCB algorithm run on $\Zset$ can be analyzed using similar techniques as those adopted in the analysis of the 
non-private GP-UCB algorithm run on the original dataset $\Xset$~\cite{srinivas10}, which leads to the regret bound shown in Theorem~\ref{th:main}.
\begin{remark}
	\label{rem:2}
	If $\sigma_{min}(\Xset) < \omega$, a similar upper bound on the regret can be proved with the difference that $\varepsilon_{ucb}$ specified by the curator is replaced by a different constant, 
	which, unlike $\varepsilon_{ucb}$, cannot be  set arbitrarily. 
	This results from the fact that if $\sigma_{min}(\Xset) < \omega$, Algorithm~\ref{alg:jl-dp} increases the singular values of the dataset $\Xset$ (see line $9$). 
	As a consequence, the pairwise distances between inputs are no longer approximately preserved after the transformation (see Theorem~\ref{lemma:dist}), 
	resulting in a looser regret bound (see Remark~\ref{rem:2_ext} in Appendix~\ref{sec:main_proof}).
\end{remark}
\begin{remark}
	\label{rem:const_regret}
	The presence of the constant $\varepsilon_{ucb}$ makes the regret upper bound of PO-GP-UCB slightly different from that of the original GP-UCB algorithm. 
	$\varepsilon_{ucb}$ can be viewed as controlling the trade-off between utility (BO performance) and privacy preservation (see more detailed discussion in Section~\ref{sec:subsec_analysis}). 
	In contrast, the only prior works on privacy-preserving BO by \citet{kusner15} and \citet{nguyen2018} do not provide any regret bounds.
\end{remark}
\begin{remark}
The upper bound on the simple regret $S_T$ in Theorem~\ref{th:main} indirectly depends on the DP parameter $\epsilon$: 
the bound holds when $\sigma_{min}(\Xset) \ge \omega$, in which $\omega$ depends on $\epsilon$ (line $5$ of Algorithm~\ref{alg:jl-dp}). 
Moreover, when $\sigma_{min}(\Xset) < \omega$, $\varepsilon_{ucb}$ (which appears in the regret bound) is replaced by a different constant, which depends on $\epsilon$ (see Remark~\ref{rem:2}).
\end{remark}

\begin{figure*}
	\centering
	{\begin{tabular}{ccc}
			\hspace{-4.5mm} \includegraphics[width=0.33 \textwidth]{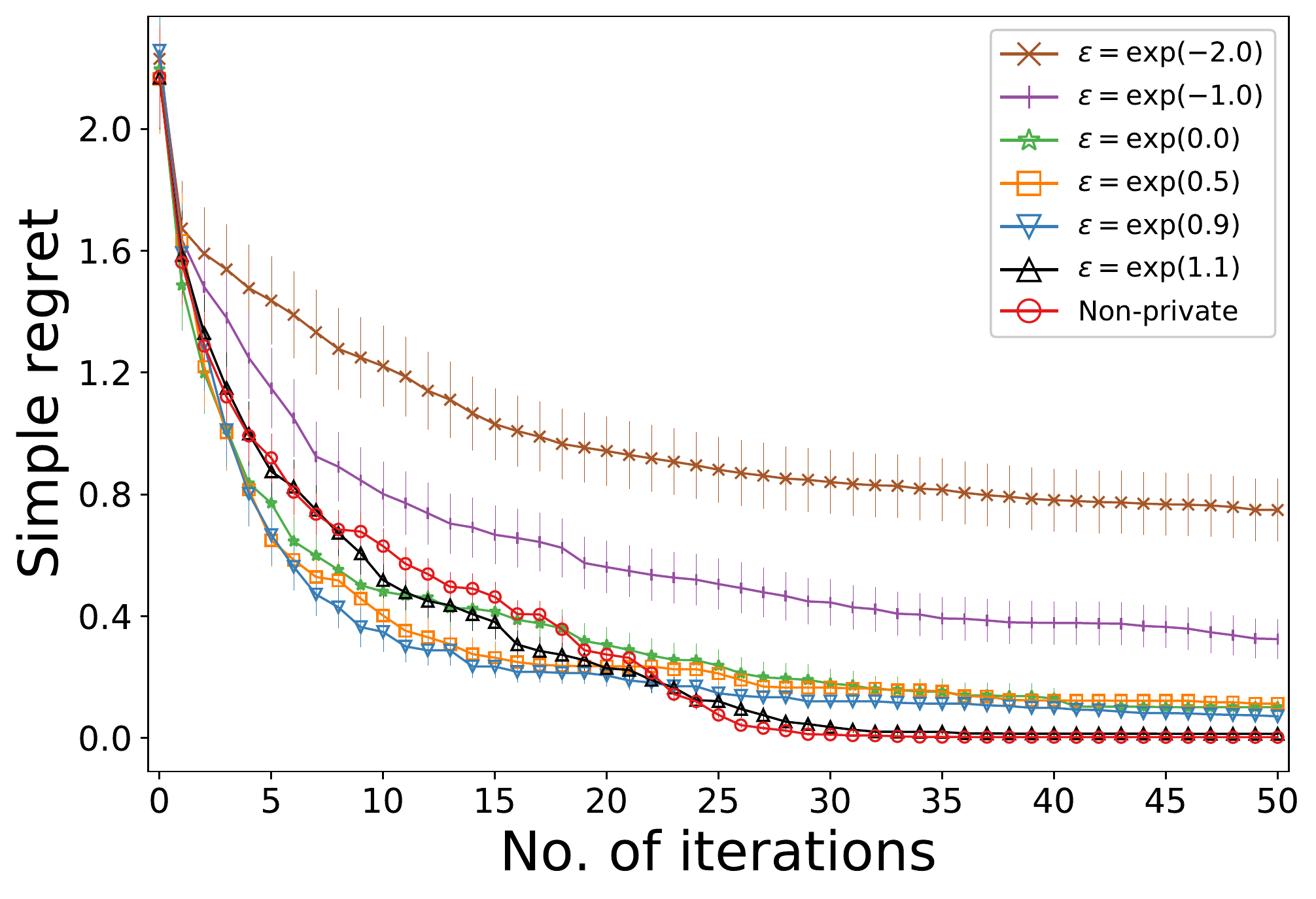} & \hspace{-4mm}
			\includegraphics[width=0.33 \textwidth]{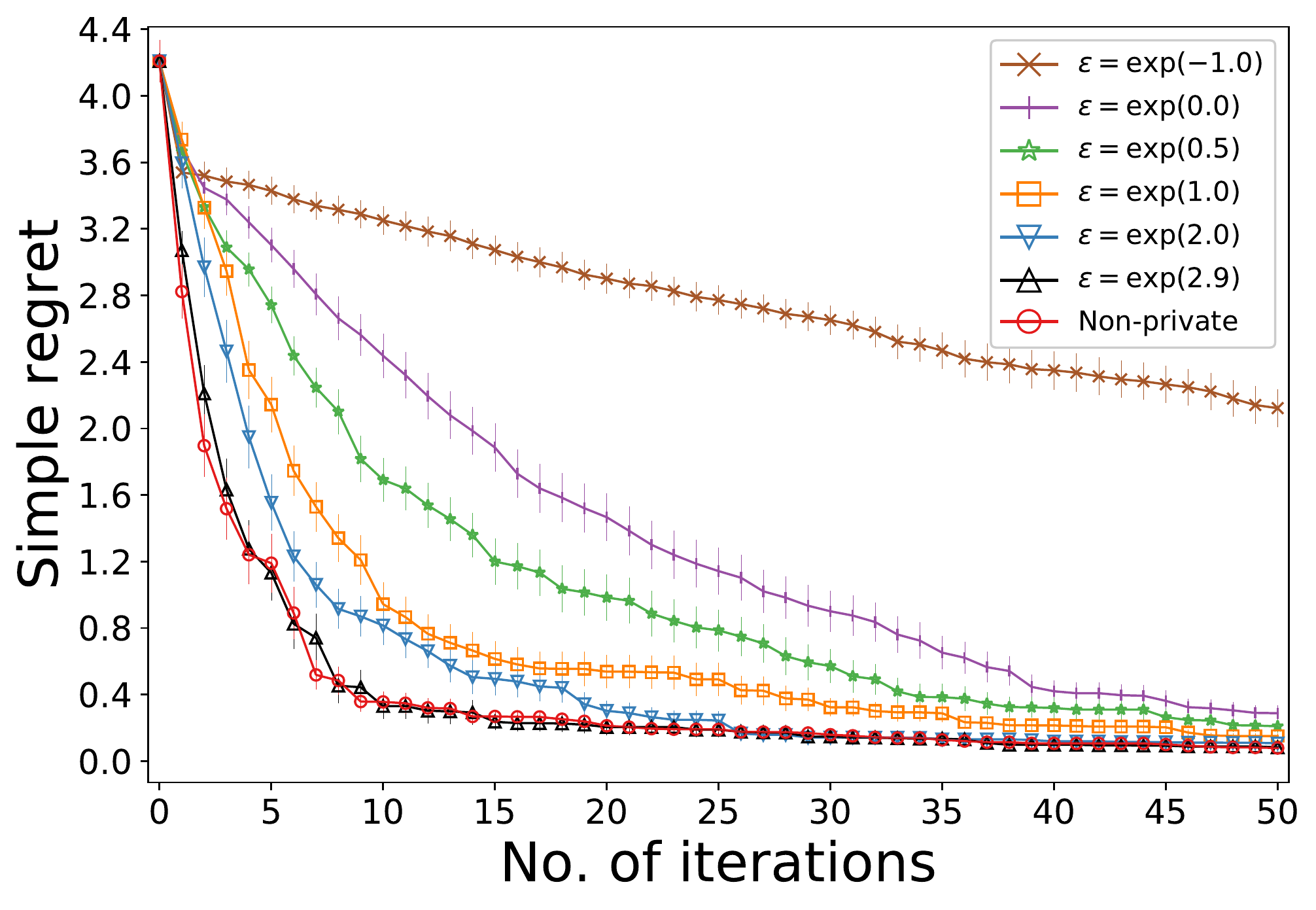} &
			\hspace{-4mm} \includegraphics[width=0.33 \textwidth]{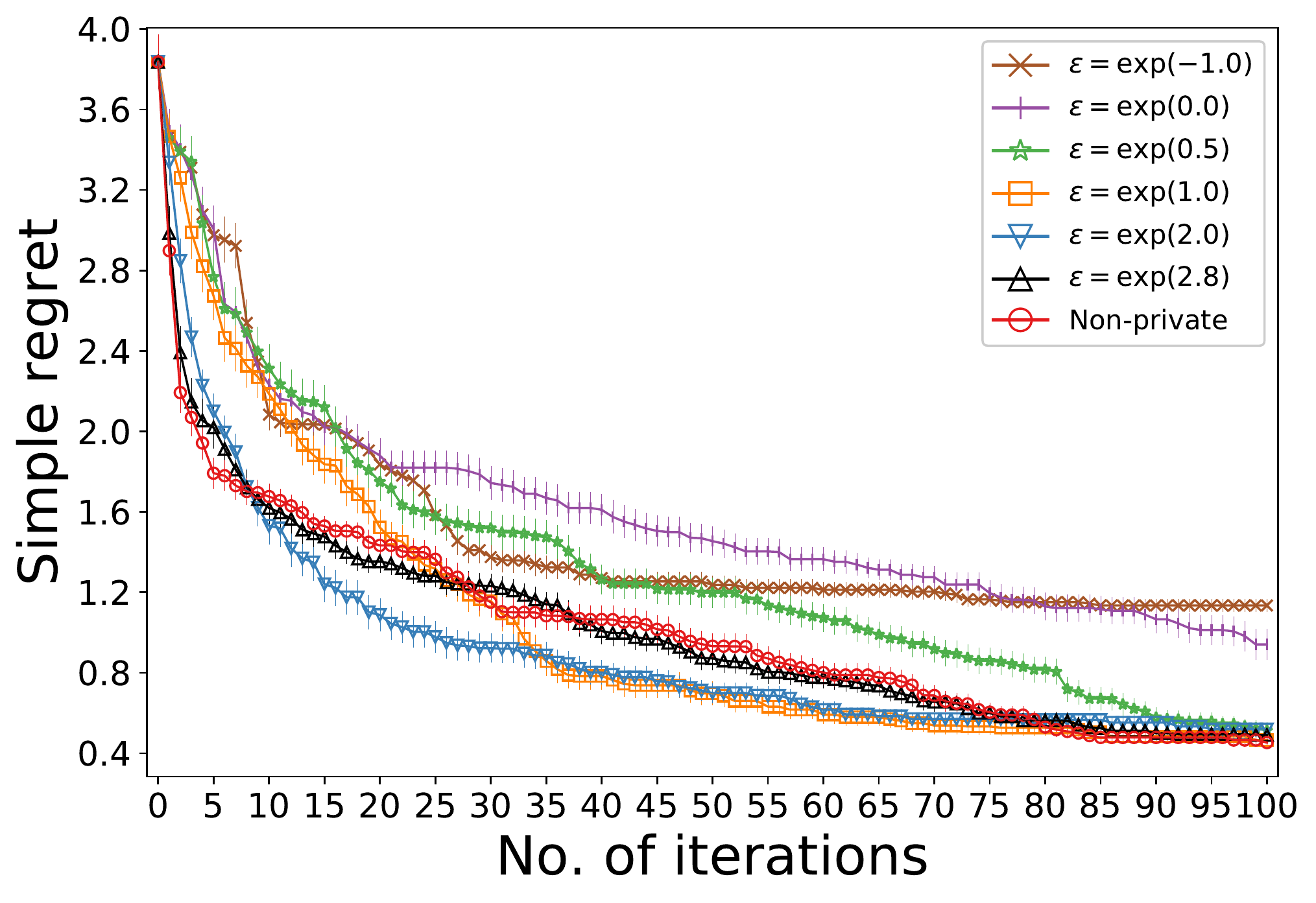} \vspace{-2mm}\\
			\hspace{-4.5mm}{\small (a)}  &\hspace{-4mm}{\small (b)}  & \hspace{-4mm} {\small (c)}
			\vspace{-2.7mm}
	\end{tabular}}
	\caption{Simple regrets achieved by tested BO algorithms (with fixed $r$ and different values of $\epsilon$) vs. the number of iterations for 
		(a) the synthetic GP dataset ($r=10$), (b) loan applications dataset ($r=15$), and (c) private property price dataset ($r=15$).}
	\label{fig:all}\vspace{-2.7mm}
\end{figure*}

\subsection{Analysis and Discussion}

\label{sec:subsec_analysis}
Interestingly, our theoretical results are amenable to elegant interpretations regarding the privacy-utility trade-off.

The flexibility to tune the value of $\omega$ to satisfy the condition required by Theorem~\ref{th:main} (i.e., $\sigma_{min}(\Xset) \ge \omega$) incurs an interesting trade-off.
Specifically, if $\sigma_{min}(\Xset) < \omega$, we can either \emph{(a)} run PO-GP-UCB without modifying any parameter, or 
\emph{(b)} reduce $\omega$ by tuning the algorithmic parameters to satisfy the condition $\sigma_{min}(\Xset) \ge \omega$, both of which incur some costs.
In case \emph{(a)}, the resulting regret bound is looser as explained in Remark~\ref{rem:2}, which might imply worse BO performance.
In case \emph{(b)}, to reduce the value of $\omega$, we can either \emph{(i)} increase the DP parameters $\epsilon$ and $\delta$ which deteriorates the DP guarantee,
or \emph{(ii)} decrease the value of $r$.
A smaller value of $r$ implies larger values of $\mu$ and $\nu$ as required by Theorem~\ref{th:main} ($r \ge 8 \log (n^2 / \mu) / \nu^2$)
and thus larger values of $\varepsilon_{ucb}$ and $\delta_{ucb}$ as seen in the definitions of $\mu$ and $\nu$ in Theorem~\ref{th:main}. 
This consequently results in a worse regret upper bound (Theorem~\ref{th:main}) and thus deteriorated BO performance.
Therefore, the privacy-utility trade-off is involved in our strategy to deal with the scenario where $\sigma_{min}(\Xset) < \omega$.

For a fixed value of $\omega$ such that $\sigma_{min}(\Xset) \ge \omega$, the privacy-utility trade-off can also be identified and thus utilized to adjust the the algorithmic parameters: $\epsilon, \delta, \varepsilon_{ucb}$ and $\delta_{ucb}$.
Specifically, decreasing the values of the DP parameters $\epsilon$ and $\delta$ improves the privacy guarantee. 
However, in order to fix the value of $\omega$ (to ensure that the condition $\sigma_{min}(\Xset) \ge \omega$ remains satisfied), 
the value of $r$ needs to be reduced, 
which results in larger values of $\varepsilon_{ucb}$ and $\delta_{ucb}$ and thus worse BO performance (as discussed in the previous paragraph).
Similar analysis reveals that decreasing the values of $\varepsilon_{ucb}$ and $\delta_{ucb}$ improves the BO performance, at the expense of looser privacy guarantee (i.e., larger required values of $\epsilon$ and $\delta$).
Furthermore, the role played by $\omega$ in Algorithm~\ref{alg:jl-dp} provides a guideline on the practical design of the algorithm.
In particular, for a fixed desirable level of privacy (i.e., fixed values of $\epsilon$ and $\delta$), the value of $r$ should be made as large as possible while still 
ensuring that the condition $\sigma_{min}(\Xset) \ge \omega$ is satisfied, since larger $r$ improves the BO performance until this condition is violated.
This guideline will be exploited and validated in the experiments.

These insights regarding the privacy-utility trade-off serve as intuitive justifications of our PO-GP-UCB algorithm
and provide useful guidelines for its practical deployment.

\section{Experiments and Discussion}

\label{sec:expt}
In this section, we empirically evaluate the performance of our PO-GP-UCB algorithm 
using four datasets including a synthetic GP dataset, a real-world loan applications dataset, a real-world property price dataset and, in Appendix~\ref{sec:branin_expt}, the Branin-Hoo benchmark function.
The performances (simple regrets) of our algorithm are compared with that of the non-private GP-UCB algorithm run using the original datasets~\cite{srinivas10}. 
The original output measurements for both real-world datasets are log-transformed to remove skewness and extremity in order to stabilize the GP covariance structure. 
The GP hyperparameters are learned using maximum likelihood estimation~\cite{gpml}.
All results are averaged over $50$ random runs, each of which uses a different set of initializations for BO.
Each random run uses an independent realization of the matrix $M$ of i.i.d. samples from $\mathcal{N}(0, 1)$ for performing random projection (line $3$ of Algorithm~\ref{alg:jl-dp}).
We set the GP-UCB parameter $\delta_{ucb} = 0.05$  (Theorem~\ref{th:main}) and normalize the  inputs  to have a maximal norm of $25$ in all experiments. Following the guidelines by the state-of-the-art works in DP~\cite{dwork14, abadi16,foulds2016,papernot2016}, we fix the value of the DP parameter $\delta$ (Definition~\ref{def:dp}) to be smaller than $1 / n$ in all experiments.
Note that setting the values of the parameters $\mu$, $\nu$ (Lemma~\ref{lemma:jl}) and the GP-UCB parameter $\varepsilon_{ucb}$ (Theorem~\ref{th:main}),
as well as assuming the diagonal dominance of covariance matrices (Definition~\ref{definition-diag}), 
is required only for our theoretical analysis and thus not necessary  in the practical employment of our algorithm.
 
In every experiment that varies the value of the DP parameter $\epsilon$ (Definition~\ref{def:dp}) (Fig.~\ref{fig:all}), 
the PO-GP-UCB algorithm with the largest value of $\epsilon$ under consideration\footnote{\label{footnote-large_eps}Further increasing the value of $\epsilon$ will only decrease the value of $\omega$ (see line $5$ of Algorithm~\ref{alg:jl-dp}), so the condition $\sigma_{min}(\Xset) \ge \omega$ will remain satisfied. As a result, the dataset  $\Zset$ returned by Algorithm~\ref{alg:jl-dp} and hence the performance of PO-GP-UCB will stay the same.}  
satisfies the condition $\sigma_{min}(\Xset) \ge \omega$  (i.e., the ``if'' clause, line $6$ of Algorithm~\ref{alg:jl-dp}), 
while the algorithms with all other values of $\epsilon$ under consideration satisfy the condition $\sigma_{min}(\Xset) < \omega$ (i.e., the ``else'' clause, line $8$ of Algorithm~\ref{alg:jl-dp}).

\noindent {\bf Synthetic GP dataset.}
The original inputs for this experiment are
$2$-dimensional vectors arranged into a uniform grid and discretized into a $100 \times 100$ input domain (i.e., $d=2$ and $n=10000$).
The function to maximize 
is sampled from a GP with the GP hyperparameters $\mu_x  = 0$, $l = 1.25$,  $\sigma_{y}^2 = 1$ and $\sigma_n^2 = 10^{-5}$.  
We set  the  parameter $r=10$ (Algorithm~\ref{alg:jl-dp}), DP parameter $\delta = 10^{-5}$ (Definition~\ref{def:dp})  and the GP-UCB parameter  $T=50$ 
for this experiment. 

Fig.~\ref{fig:all}a shows the performances of PO-GP-UCB with different values of $\epsilon$ and that of non-private GP-UCB. 
It can be observed that smaller values of $\epsilon$ (tighter privacy guarantees) result in larger simple regret, which is consistent with the privacy-utility trade off. 
PO-GP-UCB  with the largest value of $\epsilon=\exp(1.1)$ satisfying the condition $\sigma_{min}(\Xset) \ge \omega$ %
achieves only $0.011 \sigma_y$ more simple regret 
than non-private GP-UCB after $50$ iterations.
Interestingly, despite having a looser regret bound (see Remark~\ref{rem:2}), the PO-GP-UCB algorithm with some smaller values of $\epsilon$  
satisfying the condition $\sigma_{min}(\Xset) < \omega$ 
also only incurs slightly larger regret than non-private GP-UCB.
In particular, PO-GP-UCB with $\epsilon=\exp(0.9)$ ($\epsilon=\exp(0.0)$) achieves only 
$0.069 \sigma_y$ ($0.099 \sigma_y$) more simple regret 
after $50$ iterations.  
Therefore, our algorithm is able to achieve favorable performance with the values of $\epsilon$ in the single-digit range, 
which is consistent with the practice of the state-of-the-art works on the application of DP in ML~\cite{abadi16,foulds2016,papernot2016}.
This implies our algorithm's practical capability of simultaneously achieving tight privacy guarantee and obtaining competitive BO performance.

We also investigate the impact of varying the value of the random projection parameter $r$ on the performance of PO-GP-UCB. In particular, we consider $3$  different values of DP parameter $\epsilon$: $\epsilon = \exp(1.1)$, $\epsilon = \exp(1.3)$  and $\epsilon = \exp(1.5)$. 
We then fix the value of $\epsilon$ and vary the value of $r$. 
The largest value of $r$ satisfying the condition $\sigma_{min}(\Xset) \ge \omega$ is $r=10$ for $\epsilon = \exp(1.1)$, $r=15$ for $\epsilon = \exp(1.3)$ and $r=20$  for $\epsilon = \exp(1.5)$.
Tables~\ref{table:r-simulated}, ~\ref{table:r-simulated13} and~\ref{table:r-simulated15} reveal that the largest values of $r$ satisfying the condition $\sigma_{min}(\Xset) \ge \omega$ lead to the smallest simple regret after 50 iterations. 
Decreasing the value of $r$ increases the simple regret, which agrees with our analysis in Section~\ref{sec:subsec_analysis} (i.e., smaller $r$ results in worse regret upper bound). 
On the other hand,  increasing $r$ such that the condition $\sigma_{min}(\Xset) < \omega$ is satisfied 
also results in larger simple regret, which is again consistent with the analysis in Remark~\ref{rem:2} stating that the regret upper bound becomes looser in this scenario.
This experiment suggests that, in practice, for a fixed desirable privacy level (i.e., if the values of the DP parameters $\epsilon$ and $\delta$ are fixed), 
$r$ should be chosen as the largest value satisfying the condition $\sigma_{min}(\Xset) \ge \omega$. 
\begin{table}[h]
		\vspace{-3.0mm}
	\caption{Simple regrets achieved by PO-GP-UCB with fixed $\epsilon = \exp(1.1)$ and different values of $r$ after $50$ iterations for the synthetic GP dataset.  The largest value of $r$ satisfying the condition $\sigma_{min}(\Xset) \ge \omega$ is $r=10$.
 }
	\centering
	\begin{tabular}{ | c | c |c|  c | c |c| c|}
		\hline
		$r$ & $3$ & $6$& $8$ & $10$  & $15$ & $20$ \\
		\hline
		$S_{50}$ & $0.073$ & $0.038$ &  $0.018$ & $\mathbf{0.014}$ & $0.118$ & $0.137$ \\
		\hline   
	\end{tabular}
	\label{table:r-simulated}	\vspace{-3.0mm}
\end{table} 
\begin{table}[h]
		\vspace{-3.0mm}
	\caption{Simple regrets achieved by PO-GP-UCB with fixed $\epsilon = \exp(1.3)$ and different values of $r$ after $50$ iterations for the synthetic GP dataset. The largest value of $r$ satisfying the condition $\sigma_{min}(\Xset) \ge \omega$ is $r=15$.
	}
	
	\centering
	\begin{tabular}{ | c | c |c|  c | c |c| c|}
		\hline
		$r$ & $3$ & $9$& $12$ & $15$  & $20$ & $30$ \\
		\hline
		$S_{50}$ & $0.091$ & $0.009$ &  $0.019$ & $\mathbf{0.008}$ & $ 0.127$ & $ 0.134$ \\
		\hline   
	\end{tabular}
	\label{table:r-simulated13}	\vspace{-2.0mm}
\end{table} 
\begin{table}[h] 	\vspace{-2.0mm}
	\caption{Simple regrets achieved by PO-GP-UCB with fixed $\epsilon = \exp(1.5)$ and different values of $r$ after $50$ iterations for the synthetic GP dataset. The largest value of $r$ satisfying the condition $\sigma_{min}(\Xset) \ge \omega$ is $r=20$.
	}
	
	\centering
	\begin{tabular}{ | c | c |c|  c | c |c| c|}
		\hline
		$r$ & $5$ & $10$& $15$ & $20$  & $30$ & $50$ \\
		\hline
		$S_{50}$ & $0.05$ & $0.021$ &  $0.003$ & $\mathbf{0.002}$ & $0.094$ & $0.142$ \\
		\hline   
	\end{tabular}
	\label{table:r-simulated15}	\vspace{-3.0mm}
\end{table} 

\noindent {\bf Real-world loan applications dataset.}  
A bank is selecting the loan applicants with the highest return on investment (ROI)
and outsources the task to a financial AI consultancy.
The inputs to BO are the data of $36000$ loan applicants (we use the public data from \url{https://www.lendingclub.com/}), each consisting of three features: 
the total amount committed by investors for the loan, the interest rate on the loan and the annual income of the applicant (i.e., $n=36000$ and $d=3$).
The function to maximize (the output measurement) is the ROI for an applicant. 
The original inputs are preprocessed to form an isotropic covariance function\cref{footnote-isotropic}. 
We set $r=15$, $\delta = 10^{-5}$ and $T=50$.

Fig.~\ref{fig:all}b presents the results of varying the value of $\epsilon$.
Similar to the synthetic GP dataset, after 50 iterations, the simple regret achieved by PO-GP-UCB with the largest value of $\epsilon=\exp(2.9)$ satisfying the condition $\sigma_{min}(\Xset) \ge \omega$
is slightly larger (by $0.003\sigma_y$) than that achieved by non-private GP-UCB.
Moreover, PO-GP-UCB with some values of $\epsilon$ in the single-digit range satisfying the condition $\sigma_{min}(\Xset) < \omega$ shows marginally worse performance compared with non-private GP-UCB.
In particular, after 50 iterations, $\epsilon=\exp(2.0)$ and $\epsilon=\exp(1.0)$ result in $0.019 \sigma_y$ and $0.05\sigma_y$ more simple regret than non-private GP-UCB respectively.

We  examine the effect of $r$ on the performance of PO-GP-UCB, 
by fixing the value of DP parameter $\epsilon$ and changing $r$.
We consider $3$  different values of DP parameter $\epsilon$: $\epsilon = \exp(2.7)$, $\epsilon = \exp(2.9)$  and $\epsilon = \exp(3.1)$. 
The largest value of $r$ satisfying the condition $\sigma_{min}(\Xset) \ge \omega$ is $r=10$ for $\epsilon = \exp(2.7)$, $r=15$ for $\epsilon = \exp(2.9)$ and $r=20$  for $\epsilon = \exp(3.1)$.
The results are presented in Tables~\ref{table:r-loan27}, \ref{table:r-loan29} and \ref{table:r-loan31}.  PO-GP-UCB with the largest $r$ satisfying the condition $\sigma_{min}(\Xset) \ge \omega$ in general leads to the best performance, i.e., it achieves the smallest simple regret in Tables~\ref{table:r-loan27} and \ref{table:r-loan29}, and the second smallest simple regret in Table~\ref{table:r-loan31}.
Similar insights to the results of the synthetic GP dataset can also be drawn:
reducing the value of $r$ and increasing the value of $r$ to satisfy the condition $\sigma_{min}(\Xset) < \omega$
both result in larger simple regret, which again corroborates our theoretical analysis.

\begin{table}[h]
		\vspace{-2.0mm}
	\caption{Simple regrets achieved by PO-GP-UCB with fixed $\epsilon = \exp(2.7)$ and different values of $r$ after $50$ iterations for the real-world loan applications dataset. The largest value of $r$ satisfying the condition $\sigma_{min}(\Xset) \ge \omega$ is $r=10$.
	}
	
	\centering
	\begin{tabular}{ | c | c |c|  c | c |c| c|}
		\hline
		$r$ & $3$ & $6$& $8$ & $10$  & $15$ & $20$ \\
		\hline
		$S_{50}$ & $0.083$ & $0.088$ &  $0.078$ & $\mathbf{ 0.069}$ & $0.081$ & $0.076$ \\
		\hline   
	\end{tabular}
	\label{table:r-loan27}	\vspace{-3.0mm}
\end{table}
\begin{table}[h] 	
	\caption{Simple regrets achieved by PO-GP-UCB with fixed $\epsilon= \exp(2.9)$  and different values of $r$  after $50$ iterations for the real-world loan applications dataset. 
		The largest value of $r$ satisfying the condition $\sigma_{min}(\Xset) \ge \omega$ is $r=15$.
	}
	\centering
	\begin{tabular}{ | c | c |c|  c | c |c| c|}
		\hline
		$r$ & $3$ & $9$ & $12$ & $15$  & $20$ & $30$ \\
		\hline
		$S_{50}$ & $0.091$ & $0.076$ &  $0.078$ & $\mathbf{0.077}$ & $0.1$ & $0.096$ \\
		\hline
	\end{tabular}
	\label{table:r-loan29}		\vspace{-2.0mm}
\end{table}
\begin{table}[h] 	\vspace{-2.0mm}
	\caption{Simple regrets achieved by PO-GP-UCB with fixed $\epsilon = \exp(3.1)$ and different values of $r$ after $50$ iterations for the real-world loan applications dataset. The largest value of $r$ satisfying the condition $\sigma_{min}(\Xset) \ge \omega$ is $r=20$.
	}
	
	\centering
	\begin{tabular}{ | c | c |c|  c | c |c| c|}
		\hline
		$r$ & $5$ & $10$& $15$ & $20$  & $30$ & $50$ \\
		\hline
		$S_{50}$ & $0.097$ & $ 0.091$ &  $\mathbf{0.069}$ & $ 0.084$ & $0.104$ & $0.127$ \\
		\hline   
	\end{tabular}
	\label{table:r-loan31}		\vspace{-2.0mm}
\end{table}
\noindent {\bf Real-world private property price dataset.}  
A real estate agency is trying to locate the cheapest private properties and outsources the task of selecting the candidate properties to an AI consultancy. 
The original inputs are the longitude/latitude coordinates of $2004$ individual properties (i.e., $n=2004$ and $d=2$). We use the public data from \url{https://www.ura.gov.sg/realEstateIIWeb/transaction/search.action}. 
The function to minimize is the evaluated property price measured in dollars per square meter. 
We set $r=15$,  $\delta = 10^{-4}$ and $T=100$. 

The results of this experiment for different values of $\epsilon$ are displayed in Fig.~\ref{fig:all}c.
Similar observations can be made that are consistent with the previous experiments.
In particular, smaller values of $\epsilon$ (tighter privacy guarantees) generally lead to worse BO performance (larger simple regret); 
PO-GP-UCB with the largest value of $\epsilon=\exp(2.8)$ satisfying the condition $\sigma_{min}(\Xset) \ge \omega$ incurs slightly larger simple regret ($0.051 \sigma_y$) than non-private GP-UCB after 100 iterations;
PO-GP-UCB with some values of $\epsilon$   in the single-digit range satisfying the condition $\sigma_{min}(\Xset) < \omega$ exhibits small disadvantages compared with non-private GP-UCB after 100 iterations
in terms of simple regrets: $\epsilon=\exp(1.0)$ and $\epsilon=\exp(0.5)$ result in $0.017\sigma_y$ and $0.082\sigma_y$ more simple regret respectively.

We again empirically inspect the impact of $r$ on the performance of PO-GP-UCB 
in the same manner as the previous experiments: we  fix the value of $\epsilon$ and vary the value of $r$.  We consider $3$  different values of DP parameter $\epsilon$: $\epsilon = \exp(2.6)$, $\epsilon = \exp(2.8)$  and $\epsilon = \exp(3.0)$. 
The largest value of $r$ satisfying the condition $\sigma_{min}(\Xset) \ge \omega$ is $r=10$ for $\epsilon = \exp(2.6)$, $r=15$ for $\epsilon = \exp(2.8)$ and $r=20$  for $\epsilon = \exp(3.0)$.
Tables~\ref{table:r-house26}, ~\ref{table:r-house28} and ~\ref{table:r-house30} show that the smallest simple regret is achieved by the largest values of $r$ satisfying the condition $\sigma_{min}(\Xset) \ge \omega$.
Similar to the previous experiments, smaller values of $r$ and larger values of $r$ that satisfy the condition $\sigma_{min}(\Xset) < \omega$ both lead to larger simple regret,
further 
validating the practicality of our guideline on the selection of $r$ (Section~\ref{sec:subsec_analysis}).
\begin{table}[h]
	\caption{Simple regrets achieved by PO-GP-UCB with fixed $\epsilon = \exp(2.6)$ and different values of $r$ after $100$ iterations for the real-world property price dataset. The largest value of $r$ satisfying the condition $\sigma_{min}(\Xset) \ge \omega$ is $r=10$.
	}
	
	\centering
	\begin{tabular}{ | c | c |c|  c | c |c| c|}
		\hline
		$r$ & $3$ & $6$& $8$ & $10$  & $15$ & $20$ \\
		\hline
		$S_{100}$ & $0.682$ & $0.516$ &  $0.495$ & $\mathbf{0.485}$ & $\mathbf{0.485}$ & $ 0.493$ \\
		\hline   
	\end{tabular}
	\label{table:r-house26}		\vspace{-2.0mm}
\end{table}
\begin{table}[h] 	\vspace{-2.0mm}
	\caption{Simple regrets achieved by PO-GP-UCB with fixed $\epsilon = \exp(2.8)$ and different values of $r$ after $100$ iterations for the real-world property price dataset. The largest value of $r$ satisfying the condition $\sigma_{min}(\Xset) \ge \omega$ is $r=15$.
	}
	\centering
	\begin{tabular}{ | c | c |c|  c | c |c| c|}
		\hline
		$r$ & $3$ & $9$ & $12$ & $15$  & $20$ & $30$ \\
		\hline   
		$S_{100}$ & $0.567$ & $0.553$ &  $0.479$ & $\mathbf{0.453}$ & $ 0.493$ & $0.52$ \\
		\hline
	\end{tabular}
	\label{table:r-house28}	 	\vspace{-2.0mm}
\end{table} 
\begin{table}[h]
	\caption{Simple regrets achieved by PO-GP-UCB with fixed $\epsilon = \exp(3.0)$ and different values of $r$ after $100$ iterations for the real-world property price dataset. The largest value of $r$ satisfying the condition $\sigma_{min}(\Xset) \ge \omega$ is $r=20$.
	}
	
	\centering
	\begin{tabular}{ | c | c |c|  c | c |c| c|}
		\hline
		$r$ & $5$ & $10$& $15$ & $20$  & $30$ & $50$ \\
		\hline
		$S_{100}$ & $0.591$ & $0.523$ &  $0.486$ & $\mathbf{0.482}$ & $0.489$ & $0.488$ \\
		\hline   
	\end{tabular}
	\label{table:r-house30}		\vspace{-2.0mm}
\end{table}
\section{Conclusion and Future Work}
\label{sec:conclusion}
This paper describes PO-GP-UCB, 
which is the first algorithm for BO with DP in the outsourced setting with theoretical performance guarantee.
We prove the privacy-preserving property of our algorithm and show a  theoretical upper bound on the regret.
We use both synthetic and real-world experiments to show the empirical effectiveness of our algorithm, as well as its ability to achieve state-of-the-art privacy guarantees (in the single-digit range) and handle the privacy-utility trade-off.
For future work, it would be interesting to investigate whether PO-GP-UCB can be extended for privately releasing 
the output measurements $y_t$. 
To this end, the work of~\citet{hall2013} which provides a way for DP release of functional data can potentially be applied.
Another direction would be to investigate whether the work of~\citet{kenthapadi13} on DP random projection can be used as a privacy-preserving mechanism
in our outsourced BO framework to improve the privacy guarantee.
We will consider generalizing PO-GP-UCB to nonmyopic BO~\citep{dmitrii20a,ling16}, 
batch BO~\citep{daxberger17}, high-dimensional BO~\citep{NghiaAAAI18}, and multi-fidelity BO~\citep{yehong17,ZhangUAI19} settings and incorporating early stopping~\citep{dai2019} and recursive reasoning~\citep{dai2020}.
We will also consider our outsourced setting in the active learning context~\citep{LowAAMAS13,LowECML14b,hoang14,LowAAMAS08,LowICAPS09,LowAAMAS11,LowAAMAS12,LowDyDESS14,LowAAMAS14,YehongAAAI16}.
For applications with a huge budget of function evaluations, we like to couple PO-GP-UCB with the use of distributed/decentralized~\citep{LowUAI12,Chen13,LowRSS13,LowTASE15,HoangICML16,NghiaAAAI19,HoangICML19,low15,Ruofei18} or online/stochastic~\citep{NghiaICML16,MinhAAAI17,LowECML14a,LowAAAI14,teng20,Haibin19,HaibinAPP} sparse GP models to represent the belief of the unknown objective function efficiently.

\section*{Acknowledgements}

This research/project is supported by the National Research Foundation, Singapore under its Strategic Capability Research Centres Funding Initiative and its AI Singapore Programme (Award Number: AISG-GC-$2019$-$002$). Any opinions, findings, and conclusions or recommendations expressed in this material are those of the author(s) and do not reflect the views of National Research Foundation, Singapore.

\bibliography{draft}

\bibliographystyle{icml2020}

\clearpage
\if \myproof1
\clearpage
\appendix
\onecolumn

\section{Additional experimental results on Branin-Hoo function}
\label{sec:more-expt}

\label{sec:branin_expt}
In this section, we empirically evaluate the performance of our PO-GP-UCB algorithm 
using the dataset sampled from Branin-Hoo benchmark function\footnote{\url{https://www.sfu.ca/~ssurjano/branin.html}.}. The original inputs for this experiment are
$2$-dimensional vectors arranged into a uniform grid and discretized into a $31 \times 31$ input domain (i.e., $d=2$ and $n=961$).
The function to maximize 
is sampled from the negation of Branin-Hoo function. The original output measurements  are log-transformed to remove skewness and extremity in order to stabilize the GP covariance structure. 
The GP hyperparameters are learned using maximum likelihood estimation~\cite{gpml}.
Similarly to the real-world loan applications dataset in Section~\ref{sec:expt}, the original inputs are preprocessed to form an isotropic covariance function\cref{footnote-isotropic}.
All results are averaged over $50$ random runs, each of which uses a different set of initializations for BO.
We set the GP-UCB parameter $\delta_{ucb} = 0.05$  (Theorem~\ref{th:main}) and normalize the  inputs  to have a maximal norm of $25$.
We set  the  parameter $r=10$ (Algorithm~\ref{alg:jl-dp}), DP parameter $\delta = 10^{-3}$ (Definition~\ref{def:dp})  and the GP-UCB parameter  $T=50$ for this experiment.

Fig.~\ref{fig:branin} shows the performances of PO-GP-UCB with different values of $\epsilon$ and that of non-private GP-UCB. The results are consistent with the previous experiments.
Smaller values of $\epsilon$ (tighter privacy guarantees) generally lead to larger simple regret;
PO-GP-UCB with the largest value of $\epsilon=\exp(2.3)$ satisfying the condition $\sigma_{min}(\Xset) \ge \omega$ incurs only $0.004 \sigma_{y}$ more simple regret than non-private GP-UCB after $50$ iterations;
PO-GP-UCB with some values of $\epsilon$  in the single-digit range satisfying the condition $\sigma_{min}(\Xset) < \omega$ exhibits small difference in simple regret compared with non-private GP-UCB after $50$ iterations: $\epsilon=\exp(2.0)$ and $\epsilon=\exp(1.8)$ result in $0.023 \sigma_y$ and $0.051\sigma_y$ more simple regret respectively.

\begin{figure}[h]
	\centering
			\hspace{-0.0mm} \includegraphics[width=0.5\textwidth]{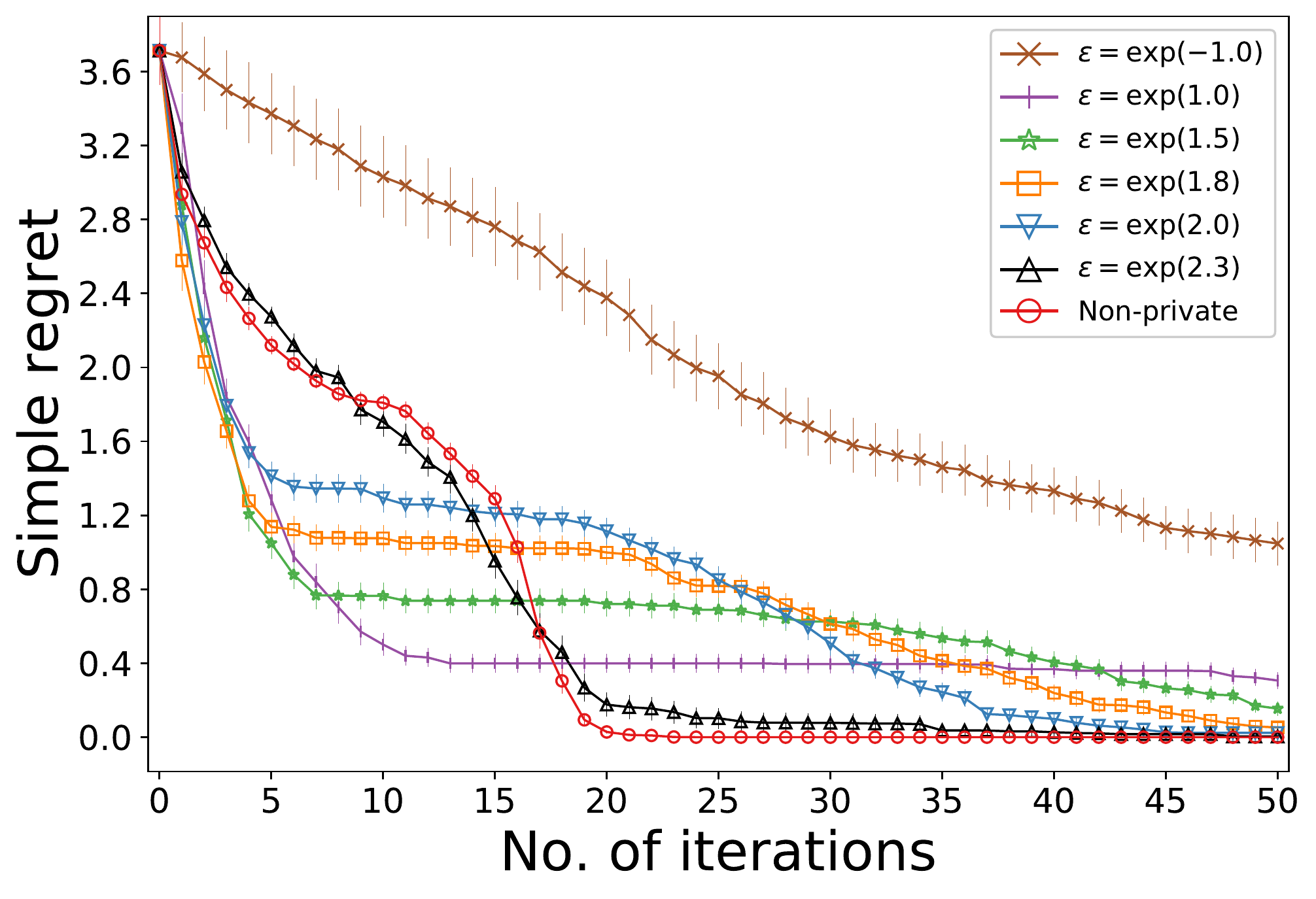} 
	\caption{Simple regrets achieved by tested BO algorithms (with fixed $r=10$ and different values of $\epsilon$) vs. the number of iterations for the Branin-Hoo function dataset.}
	\label{fig:branin}
\end{figure}

Similarly to the experiments in the main text, we  investigate the impact of varying the value of the random projection parameter $r$ on the performance of PO-GP-UCB. 
We consider $3$  different values of DP parameter $\epsilon$: $\epsilon = \exp(2.3)$, $\epsilon = \exp(2.5)$  and $\epsilon = \exp(2.7)$. 
We   fix the value of $\epsilon$ and vary the value of $r$. 
The largest value of $r$ satisfying the condition $\sigma_{min}(\Xset) \ge \omega$ is $r=10$ for $\epsilon = \exp(2.3)$, $r=15$ for $\epsilon = \exp(2.5)$ and $r=20$  for $\epsilon = \exp(2.7)$.
Tables~\ref{table:r-branin23},~\ref{table:r-branin25} and ~\ref{table:r-branin27} reveal that the largest values of $r$ satisfying the condition $\sigma_{min}(\Xset) \ge \omega$ lead to the smallest simple regret after 50 iterations. 
Decreasing the value of $r$ increases the simple regret, which agrees with our analysis in Section~\ref{sec:subsec_analysis} (i.e., smaller $r$ results in worse regret upper bound). 
Increasing $r$ such that the condition $\sigma_{min}(\Xset) < \omega$ is satisfied, on the other hand, also results in larger simple regret, which is again consistent with the analysis in Remark~\ref{rem:2} stating that the regret upper bound becomes looser in this scenario.
These observations are consisted with those for a synthetic GP dataset, a real-world loan applications dataset and a real-world property price dataset in the main text.

\begin{table}[h]
	\caption{Simple regrets achieved by PO-GP-UCB with fixed $\epsilon = \exp(2.3)$ and different values of $r$ after $50$ iterations for  the Branin-Hoo function dataset. The largest value of $r$ satisfying the condition $\sigma_{min}(\Xset) \ge \omega$ is $r=10$.
	}
	\centering
	\begin{tabular}{ | c | c |c|  c | c |c| c|}
		\hline
		$r$ & $3$ & $6$& $8$ & $10$  & $15$ & $20$ \\
		\hline
		$S_{50}$ & $0.53$ & $0.184$ &  $0.038$ & $\mathbf{0.0}$ & $0.005$ & $0.024$ \\
		\hline   
	\end{tabular}
	\label{table:r-branin23}	
\end{table}

\begin{table}[h]
	\caption{Simple regrets achieved by PO-GP-UCB with fixed $\epsilon = \exp(2.5)$ and different values of $r$ after $50$ iterations for  the Branin-Hoo function dataset. The largest value of $r$ satisfying the condition $\sigma_{min}(\Xset) \ge \omega$ is $r=15$.
	}

	\centering
	\begin{tabular}{ | c | c |c|  c | c |c| c|}
		\hline
		$r$ & $3$ & $9$& $12$ & $15$  & $20$ & $30$ \\
		\hline
		$S_{50}$ & $0.259$ & $0.001$ &  $\mathbf{0.0}$ & $\mathbf{0.0}$ & $0.014$ & $ 0.026$ \\
		\hline   
	\end{tabular}
	\label{table:r-branin25}	
\end{table}

\begin{table}[h]
	\caption{Simple regrets achieved by PO-GP-UCB with fixed $\epsilon = \exp(2.7)$ and different values of $r$ after $50$ iterations for  the Branin-Hoo function dataset. The largest value of $r$ satisfying the condition $\sigma_{min}(\Xset) \ge \omega$ is $r=20$.
	}
	\centering
	\begin{tabular}{ | c | c |c|  c | c |c| c|}
		\hline
		$r$ & $5$ & $10$& $15$ & $20$  & $30$ & $50$ \\
		\hline
		$S_{50}$ & $0.152$ & $\mathbf{0.0}$ &  $\mathbf{0.0}$ & $\mathbf{0.0}$ & $0.005$ & $0.073$ \\
		\hline   
	\end{tabular}
	\label{table:r-branin27}	
\end{table}

\section{Proofs and derivations}

\subsection{Proof of Lemma~\ref{lemma:jl}}
\label{app:rem1}
\begin{theorem}
	\label{th:jl}[Johnson-Lindenstrauss lemma~\cite{johnson84}]
	Let $\nu \in (0, 1/2)$, $r\in\mathbb{N}$ and $d \in\mathbb{N}$ be given. Let $M'$ be a $r \times d$ matrix whose entries are i.i.d. samples from $\mathcal{N}(0, 1)$. Then for any vector $y \in \mathbb{R}^d$
	\begin{equation*}
	P\Big( (1 - \nu) \lVert y \rVert^2  \leq r^{-1} \lVert M' y \rVert^2  \le (1 + \nu) \lVert y \rVert^2 \Big) \ge 1 - 2 \exp(- \nu^2 r/8).
	\end{equation*}
\end{theorem}

\noindent \textit{Proof of lemma.} Fix $x, x' \in \Xset$. It follows from  Theorem~\ref{th:jl} by setting vector $y= (x - x')^\top$ and  $r \times d$ matrix $M' = M^\top$  that
\begin{equation}
\begin{array}{l}
\displaystyle 1 - 2 \exp(- \nu^2 r/8) \\
\displaystyle \le P\Big( (1 - \nu) \lVert (x - x')^\top \rVert^2  \leq r^{-1} \lVert M^\top (x - x')^\top \rVert^2  \le (1 + \nu) \lVert  (x - x')^\top  \rVert^2 \Big)  \\
\displaystyle = P\Big( (1 - \nu) \lVert x - x' \rVert^2  \leq r^{-1} \lVert  x M -  x' M \rVert^2  \le (1 + \nu) \lVert x - x' \rVert^2 \Big). 
\end{array}
\label{eq:app1}
\end{equation}
Since there are no more than $n^2/2$ pairs of inputs $x, x' \in \Xset$, applying the union bound to~\eqref{eq:app1} gives that the probability of 
\begin{equation*}
(1 - \nu) \lVert x - x' \rVert^2  \leq r^{-1} \lVert  x M -  x' M \rVert^2  \le (1 + \nu) \lVert x - x' \rVert^2 
\end{equation*}
for all  $x, x' \in \Xset$ is at least $ 1 - n^2 \exp(- \nu^2 r/8)$.

To guarantee that the probability of $(1 - \nu) \lVert x - x' \rVert^2  \leq r^{-1} \lVert M x - M x' \rVert^2  \le (1 + \nu) \lVert x - x' \rVert^2$ for all  $x, x' \in \Xset$ is at least $1 - \mu$, the value of $r$ has to satisfy the following inequality:
\begin{equation*}
1 - n^2 \exp(- \nu^2 r/8) \ge 1 - \mu, 
\end{equation*}
which is equivalent to $r \ge 8 \log (n^2 / \mu) / \nu^2$.

\subsection{Privacy guarantee of Algorithm~\ref{alg:jl-dp}}
\label{app:th-jl-dp}
\subsubsection{Comparison between Algorithm~\ref{alg:jl-dp} and Algorithm~$3$ of~\citet{blocki2012}}
There are several important differences between our Algorithm~\ref{alg:jl-dp} and the work of~\citet{blocki2012}. Firstly, Algorithm~$3$ of~\citet{blocki2012} outputs a DP estimate $r^{-1} \tilde{\Xset }^\top M^\top M \tilde{\Xset }$ (in the notations of Algorithm~\ref{alg:jl-dp}) of the covariance matrix $r^{-1} \Xset ^\top \Xset$, while our Algorithm~\ref{alg:jl-dp} outputs a DP transformation $r^{-1/2} \Xset  M$ (or $r^{-1/2} \tilde{\Xset } M$) of the original dataset $\Xset$. However, the authors of~\citet{blocki2012} prove the privacy guarantee (see Theorem $4.1$, p. 13 of their paper) by showing that releasing $\tilde{\Xset }^\top M^\top$ (using matrix $M$ of size $r\times n$) preserves DP and then apply the post-processing property of DP to  reconstruct $r^{-1} \tilde{\Xset }^\top M^\top M \tilde{\Xset }$. This observation allows us to modify their proof for our Algorithm~\ref{alg:jl-dp}. Additionally, matrix $\tilde{\Xset }^\top M^\top$ (in the notations of Algorithm~\ref{alg:jl-dp})  in the proof of~\citet{blocki2012} has size $d\times r$, while matrices $r^{-1/2} \Xset  M$ and $r^{-1/2} \tilde{\Xset } M$ returned by our Algorithm~\ref{alg:jl-dp} have size $n \times r$, which requires us to modify the proof of~\citet{blocki2012}. These modifications are discussed in Section~\ref{app:th-jl-dp-proof} below.

Secondly, Algorithm~$3$ of~\citet{blocki2012} does not have the ``if/else'' condition (line $6$ of Algorithm~\ref{alg:jl-dp}) and always increases the singular values as in line $9$ of Algorithm~\ref{alg:jl-dp}, since the authors are able to offset the bias introduced to the estimate of covariance of the dataset  along a given dimension by increasing the singular values. Specifically, they do it by subtracting $\omega^2$ from the computed estimate (see Algorithm~$4$ in~\citet{blocki2012}). For our case, however,  the distances between the  original inputs from the dataset  $\Xset$ are no longer approximately the same as the distances between their images from the dataset  $\Zset$ when $\sigma_{min}(\Xset) < \omega$ (i.e., the ``else'' clause, line $8$ of Algorithm~\ref{alg:jl-dp}), as shown in Theorem~\ref{lemma:dist}. Therefore, the case of $\sigma_{min}(\Xset) < \omega$ results  in a slightly different regret bound (see Theorem~\ref{th:main} and Remark~\ref{rem:2}) and requires us to introduce  the ``if/else'' condition into Algorithm~\ref{alg:jl-dp}.
Introducing such an ``if/else'' condition, however, does not affect the  proof of Theorem $4.1$ of~\citet{blocki2012} and our proof: the ``if'' clause (line $6$ of Algorithm~\ref{alg:jl-dp}) is stated in the Corollary (see p. $17$ of~\citet{blocki2012}), while the ``else'' clause (line $8$ of Algorithm~\ref{alg:jl-dp})  is proved in Theorem $4.1$ of~\citet{blocki2012}.

\subsubsection{Proof of Theorem~\ref{th:jl-dp}}
\label{app:th-jl-dp-proof}
Fix two neighboring datasets $\Xset $ and  $\Xset '$. Let $E \triangleq \Xset ' - \Xset $, such that $E$ is a rank $1$ matrix. Without  loss of generality, we assume that in the definition of neighboring datasets (Definition~\ref{def:neighb})  $\Vert x^{}_{(i^*)} - x'_{(i^*)} \rVert = 1$. Then  we can write  $E$ as the outer product $E = e_{i^*} v^\top$ where $e_{i^*}$ is the indicator vector of row $i^*$ and $v$ is the vector of norm $1$. Then the singular values of $E$ are exactly $\{1, 0,\ldots, 0\}$ (see~\citet{blocki2012}, p. $14$).

Similar to  Theorem $4.1$ of~\citet{blocki2012}, the proof is composed of two stages. For the first stage we work under the premise that both  and $\Xset $ and  $\Xset '$ have singular values no less than $\omega$ (the ``if'' clause, line $6$ of Algorithm~\ref{alg:jl-dp}). For the second stage we denote $\tilde{\Xset }$ and  $\tilde{\Xset }'$ as the respective matrices from  ``else'' clause (line $8$ of Algorithm~\ref{alg:jl-dp}) and show what adaptations are needed to make the proof follow through.

We prove the theorem for the scaled  output of the ``if'' clause of Algorithm~\ref{alg:jl-dp} $\Xset  M$ (the post-processing property of DP can be applied after that to  reconstruct $r^{-1/2}  \Xset  M$). $ \Xset  M$ is composed of $r$ columns each is an i.i.d. sample from $\Xset  Y$ where $Y  \sim \mathcal{N}(0, I_{d \times d})$. The following lemma is similar to Claim $4.3$ of~\citet{blocki2012}(p. $14$):
\begin{lemma}
	\label{lemma:blocki}
	Let  $\epsilon > 0$, $\delta \in (0,1)$, $r\in\mathbb{N}$, $d \in\mathbb{N}$,  two neighboring datasets $\Xset $ and $\Xset '$  and $Y$ sampled  from $\mathcal{N}(0, I_{d \times d})$ be given.
	Fix $\epsilon_0 \triangleq \epsilon / \sqrt{4 r \log(2 / \delta)}$ and $\delta_0 \triangleq  \delta / (2 r)$. Denote 
	\begin{equation*}
	S \triangleq \{\xi \in \mathbb{R}^n: \exp(-\epsilon_0)  \text{PDF}_{\Xset 'Y}(\xi) 
	\le \text{PDF}_{\Xset Y}(\xi) 
	\le \exp(\epsilon_0)  \text{PDF}_{\Xset 'Y}(\xi) \}
	\end{equation*}
	where PDF is the probability density function. Then $P(S) \ge 1 - \delta_0$.
\end{lemma}
\begin{proof}
	Similar to the proof of Claim $4.3$ of~\citet{blocki2012}, first we formally define the PDF of the two distributions. We apply the fact that $\Xset Y$ and $\Xset 'Y$ are linear transformations of $\mathcal{N}(0, I_{d \times d})$.
	\begin{equation*}
	\begin{array}{l}
	\displaystyle   \text{PDF}_{\Xset Y}(\xi)  = 
	\frac{1}{\sqrt{(2 \pi)^n \det (\Xset \Xset ^\top)}} \exp\Big(-\frac{1}{2} \xi^\top (\Xset \Xset ^\top)^{-1} \xi \Big) \\
	\displaystyle   \text{PDF}_{\Xset 'Y}(\xi) =  \frac{1}{\sqrt{(2 \pi)^n \det (\Xset '\Xset '^\top)}} \exp\Big(-\frac{1}{2} \xi^\top (\Xset '\Xset '^\top)^{-1} \xi \Big).
	\end{array}
	\end{equation*}
	If the matrix $\Xset \Xset^\top$  (all the reasoning here is exactly the same for $\Xset '\Xset '^\top$)  is not full-rank, the SVD allows us to use similar notation to denote the generalizations of the inverse and of the determinant: The Moore-Penrose inverse of any square matrix $M$ is $M^\dagger \triangleq V \Sigma^{-1} U^\top$ where $M = U \Sigma V^\top$ is the SVD of matrix $M$, and the pseudo-determinant of $M$ is $\widetilde{det}(M) \triangleq
	\Pi_{i=1}^{rank(M)} \sigma_i(M)$ where $\sigma_{i}(M)$ are the singular values of matrix $M$.
	Furthermore, if $\Xset \Xset^\top$ has non-trivial kernel space (i.e., is not invertible)  then $\text{PDF}_{\Xset Y}$  in the equation above is technically undefined. However, if we restrict ourselves only to the subspace $\mathcal{V} = (\text{Ker}(\Xset \Xset^\top))^\perp$,  then $\text{PDF}^{\mathcal{V}}_{\Xset Y}$ is defined over $\mathcal{V}$ and $\text{PDF}^{\mathcal{V}}_{\Xset Y}(\xi)  \triangleq 
	\frac{1}{\sqrt{(2 \pi)^{rank(\Xset \Xset ^\top)} \widetilde{\det} (\Xset \Xset ^\top)}} \exp\Big(-\frac{1}{2} \xi^\top (\Xset \Xset ^\top)^{\dagger} \xi \Big) \\$ From now on, we omit 
	the superscript from the PDF and refer to the above function as the PDF of $\Xset Y$.	
	See p. $4$--$5$ of \citet{blocki2012} for more details.

	Similar to the proof of Claim $4.3$ of~\citet{blocki2012}, first we show that 
	\begin{equation*}
	\exp(-\epsilon_0 / 2)  \le \sqrt \frac{\det (\Xset '\Xset '^\top)}{\det (\Xset  \Xset ^\top)} \le \exp(\epsilon_0 / 2). 
	\end{equation*}
	The proof copies the derivation of eq. $4$ in~\citet{blocki2012} (p. $15$) with replacing $A$ to $\Xset ^\top$, $A'$ to $\Xset '^\top$,  $x$ to $\xi$ and swapping $n$ and $d$ where necessary.
	
	Next we prove an analogue of eq. $5$ of Claim $4.3$ of~\citet{blocki2012}:
	\begin{equation}
	\label{eq:stage1b}
	P_\xi \Bigg(\frac{1}{2}|  \xi^\top \big((\Xset \Xset ^\top)^{-1} -  (\Xset '\Xset '^\top)^{-1} \big) \xi | \ge \epsilon_0/2 \Bigg) \le \delta_0.
	\end{equation}
	To do this:
	\begin{equation}
	\label{eq:14}
	\begin{array}{l}
	\displaystyle  \xi^\top \big((\Xset \Xset ^\top)^{-1} -  (\Xset '\Xset '^\top)^{-1} \big) \xi  \\
	\displaystyle  =  \xi^\top \big((\Xset \Xset ^\top)^{-1} -  (\Xset '\Xset '^\top)^{-1} \Xset \Xset ^\top (\Xset \Xset ^\top)^{-1}  \big) \xi \\
	\displaystyle  =  \xi^\top \big((\Xset \Xset ^\top)^{-1} -  (\Xset '\Xset '^\top)^{-1} (\Xset ' - E)( \Xset ' - E)^\top (\Xset \Xset ^\top)^{-1}  \big) \xi \\
	\displaystyle  =  \xi^\top \big((\Xset \Xset ^\top)^{-1} -  (\Xset '\Xset '^\top)^{-1} (\Xset '\Xset '^\top - E \Xset '^\top - \Xset 'E^\top + EE^\top) (\Xset \Xset ^\top)^{-1}  \big) \xi \\
	\displaystyle  =  \xi^\top \big((\Xset \Xset ^\top)^{-1} - (\Xset \Xset ^\top)^{-1} -  (\Xset '\Xset '^\top)^{-1} ( - E \Xset '^\top - \Xset 'E^\top + EE^\top) (\Xset \Xset ^\top)^{-1}  \big) \xi \\
	\displaystyle  =  \xi^\top (\Xset '\Xset '^\top)^{-1} (  E \Xset '^\top + \Xset 'E^\top - EE^\top) (\Xset \Xset ^\top)^{-1}  \xi \\
	\displaystyle  =  \xi^\top (\Xset '\Xset '^\top)^{-1} (  E \Xset ^\top + \Xset 'E^\top) (\Xset \Xset ^\top)^{-1}  \xi
	\end{array}
	\end{equation}
	where the second and the last equalities are due to $E = \Xset ' - \Xset $. The expression in the last line of~\eqref{eq:14} is very similar to the one in the derivation of eq. $5$ in~\citet{blocki2012} (p. $15$). The difference is that in order for the proof to go through, we need to multiply $ (\Xset '\Xset '^\top)^{-1}$ by $\Xset \Xset ^\top (\Xset \Xset ^\top)^{-1}$ in the second line of~\eqref{eq:14}, while the original proof of~\citet{blocki2012} multiplies  $ (\Xset^\top \Xset)^{-1}$ by $\Xset '^\top \Xset ' (\Xset '^\top \Xset ')^{-1}$ (in our notations), see eq. in the bottom of p. $15$ of~\citet{blocki2012}.

	Now denoting singular value decompositions of $\Xset  = U \Sigma V^\top$ and $\Xset ' = U' \Lambda V'^\top$, and the fact that $E=e_{i^*} v^\top$, we continue~\eqref{eq:14}:
	\begin{equation}
	\label{eq:15}
	\begin{array}{l}
	\displaystyle  \xi^\top (\Xset '\Xset '^\top)^{-1} (  E \Xset ^\top + \Xset 'E^\top) (\Xset \Xset ^\top)^{-1}  \xi \\
	\displaystyle   =  \xi^\top (\Xset '\Xset '^\top)^{-1}  E \Xset ^\top (\Xset \Xset ^\top)^{-1}  \xi  
	+  \xi^\top (\Xset '\Xset '^\top)^{-1} \Xset 'E^\top (\Xset \Xset ^\top)^{-1}  \xi \\
	\displaystyle  = \xi^\top (U' \Lambda V'^\top V' \Lambda U'^\top)^{-1} ( e_{i^*} \cdot v^\top V \Sigma U^\top)  (U \Sigma V^\top V \Sigma U^\top)^{-1}  \xi  \\
	\displaystyle  + \xi^\top (U' \Lambda V'^\top V' \Lambda U'^\top)^{-1} ( U' \Lambda V'^\top  v\cdot e_{i^*}^\top)  (U \Sigma V^\top V \Sigma U^\top)^{-1}  \xi \\
	\displaystyle  = \xi^\top U' \Lambda^{-2}  U'^\top  e_{i^*} \cdot v^\top V  \Sigma^{-1} U^\top  \xi  + \xi^\top U' \Lambda^{-1} V'^\top  v\cdot e_{i^*}^\top U \Sigma^{-2} U^\top  \xi \\
	\end{array}
	\end{equation}
	where the last equality is due to the properties of singular value decomposition.
	
	So now, assume $\xi$ is sampled from $\Xset 'Y$ (the case of $\Xset Y$ is symmetric). That is, assume that we've sampled $\chi$ from $Y \sim \mathcal{N}(0, I_{d \times d})$ and we have $\xi = \Xset '\chi = U' \Lambda V'^\top \chi$ and equivalently $\xi = (\Xset  + E) \chi = U \Sigma V^\top \chi + e_{i^*} v^\top \chi$. Plugging it into~\eqref{eq:15} gives:
	\begin{equation*}
	\begin{array}{l}
	\displaystyle   | \xi^\top U' \Lambda^{-2}  U'^\top  e_{i^*} 
	\cdot v^\top V  \Sigma^{-1} U^\top  \xi  
	+ \xi^\top U' \Lambda^{-1} V'^\top  v
	\cdot e_{i^*}^\top U \Sigma^{-2} U^\top  \xi |  \\
	\displaystyle  = | (U' \Lambda V'^\top \chi)^\top U' \Lambda^{-2}  U'^\top  e_{i^*} 
	\cdot v^\top V  \Sigma^{-1} U^\top  (U \Sigma V^\top \chi + e_{i^*} v^\top \chi)  \\
	\displaystyle  + (U' \Lambda V'^\top \chi)^\top U' \Lambda^{-1} V'^\top  v
	\cdot e_{i^*}^\top U \Sigma^{-2} U^\top  (U \Sigma V^\top \chi + e_{i^*} v^\top \chi) |  \\
	\displaystyle  = | \chi^\top V' \Lambda U'^\top U' \Lambda^{-2}  U'^\top  e_{i^*} 
	\cdot v^\top V  \Sigma^{-1} U^\top  (U \Sigma V^\top \chi + e_{i^*} v^\top \chi)  \\
	\displaystyle  + \chi^\top V' \Lambda U'^\top U' \Lambda^{-1} V'^\top  v
	\cdot e_{i^*}^\top U \Sigma^{-2} U^\top  (U \Sigma V^\top \chi + e_{i^*} v^\top \chi) |  \\
	\displaystyle  \le {term}_1 \cdot {term}_2 + {term}_3 \cdot {term}_4
	\end{array}
	\end{equation*}
	where for $i=1,2,3,4$ we have ${term}_i = | {vec}_i \cdot \chi |$ and 
	\begin{equation*}
	\begin{array}{l}
	\displaystyle  vec_{1}  \\
	\displaystyle  = (V' \Lambda U'^\top U' \Lambda^{-2} U'^\top e_{i^*})^\top \\
	\displaystyle  = (V'\Lambda^{-1} U'^\top e_{i^*})^\top
	\end{array}
	\end{equation*}
	so $\rVert vec_1 \rVert \le  1 / \lambda_d$; 
	\begin{equation*}
	\begin{array}{l}
	\displaystyle  vec_{2}  \\
	\displaystyle  = v^\top V  \Sigma^{-1} U^\top  (U \Sigma V^\top + e_{i^*} v^\top) \\
	\displaystyle  = v^\top + v^\top V  \Sigma^{-1} U^\top e_{i^*} v^\top 
	\end{array}
	\end{equation*}
	so $\rVert vec_2 \rVert \le 1 + 1 / \sigma_d$; 
	\begin{equation*}
	\begin{array}{l}
	\displaystyle  vec_{3}  \\
	\displaystyle  = (V' \Lambda U'^\top U'  \Lambda^{-1} V'^\top  v)^\top \\
	\displaystyle  = v^\top 
	\end{array}
	\end{equation*}
	so $\rVert vec_3 \rVert \le 1$;
	
	\begin{equation*}
	\begin{array}{l}
	\displaystyle  vec_{4}  \\
	\displaystyle  =e_{i^*}^\top U \Sigma^{-2} U^\top  (U \Sigma V^\top + e_{i^*} v^\top) \\
	\displaystyle  = e_{i^*}^\top U \Sigma^{-1} V^\top + e_{i^*}^\top U \Sigma^{-2} U^\top e_{i^*} v^\top
	\end{array}
	\end{equation*}
	so $\rVert vec_4 \rVert \le 1/ \sigma_d + 1 / \sigma^2_d$ where $\sigma_d$ and $\lambda_d$ are the smallest singular values of $\Xset $ and $\Xset '$, respectively. The remainder of the proof now follows the proof of Claim $4.3$ of~\citet{blocki2012} with replacing $A$ to $\Xset ^\top$, $A'$ to $\Xset '^\top$,  $x$ to $\xi$ and swapping $n$ and $d$ where necessary.
\end{proof}
For the second stage we assume that  ``else'' clause (line $8$ of Algorithm~\ref{alg:jl-dp}) is applied and denote $\tilde{\Xset } \triangleq U \sqrt{\Sigma^2  + \omega^2 I_{n \times d}} V^\top$ and  $\tilde{\Xset }' \triangleq U' \sqrt{\Lambda^2  + \omega^2 I_{n \times d}} V'^\top$. The theorem requires an analogue of Lemma~\ref{lemma:blocki} to hold, which depends on the following two conditions:

\begin{equation}
\label{eq:stage2a}
\exp(-\epsilon_0 / 2) \le \sqrt \frac{\det (\tilde{\Xset }' \tilde{\Xset }'^\top)}{\det (\tilde{\Xset } \tilde{\Xset }^\top)} \le \exp(\epsilon_0 / 2). 
\end{equation}
\begin{equation}
\label{eq:stage2b}
P_\xi \Bigg(\frac{1}{2}|  \xi^\top \big((\tilde{\Xset }\tilde{\Xset }^\top)^{-1} -  (\tilde{\Xset }'\tilde{\Xset }'^\top)^{-1} \big) \xi | \ge \epsilon_0/2 \Bigg) \le \delta_0.
\end{equation}
Derivation of~\eqref{eq:stage2a} copies the derivation of eq. $6$ in~\citet{blocki2012} (p. $16$). To derive~\eqref{eq:stage2b}, we start with an observation regarding $\Xset '\Xset '^\top$ and $\tilde{\Xset }'\tilde{\Xset }'^\top$:
\begin{equation}
\label{eq:a-b-comp}
\begin{array}{l}
\displaystyle  \Xset '\Xset '^\top = (\Xset + E) (\Xset  + E)^\top= 	\Xset \Xset ^\top  + \Xset ' E^\top + E \Xset ^\top \\
\displaystyle  \tilde{\Xset }\tilde{\Xset }^\top = U(\Sigma^2 + \omega^2 I)U^\top  = U\Sigma^2 U^\top +  \omega^2 I =  \Xset \Xset ^\top +   \omega^2 I \\
\displaystyle  \tilde{\Xset }'\tilde{\Xset }'^\top = U'(\Lambda^2 + \omega^2 I)U'^\top  = U'\Lambda^2 U'^\top +  \omega^2 I =  \Xset '\Xset '^\top +   \omega^2 I  \\
\displaystyle  \implies \tilde{\Xset }' \tilde{\Xset }'^\top - \tilde{\Xset }\tilde{\Xset }^\top =  \Xset ' E^\top + E \Xset ^\top. 
\end{array}
\end{equation}
Now we can follow the same outline as in the proof of~\eqref{eq:stage1b}. Fix $\xi$, then
\begin{equation}
\label{eq:stage2-changed}
\begin{array}{l}
\displaystyle  \xi^\top \big((\tilde{\Xset }\tilde{\Xset }^\top)^{-1} -  (\tilde{\Xset }'\tilde{\Xset }'^\top)^{-1} \big) \xi  \\
\displaystyle  =  \xi^\top \big((\tilde{\Xset }\tilde{\Xset }^\top)^{-1} -  (\tilde{\Xset }'\tilde{\Xset }'^\top)^{-1} \tilde{\Xset }\tilde{\Xset }^\top (\tilde{\Xset }\tilde{\Xset }^\top)^{-1}  \big) \xi \\
\displaystyle  =  \xi^\top \big((\tilde{\Xset }\tilde{\Xset }^\top)^{-1} -  (\tilde{\Xset }'\tilde{\Xset }'^\top)^{-1}(	\tilde{\Xset }'\tilde{\Xset }'^\top -  \Xset ' E^\top - E \Xset ^\top) (\tilde{\Xset }\tilde{\Xset }^\top)^{-1}  \big) \xi \\
\displaystyle  =  \xi^\top \big((\tilde{\Xset }\tilde{\Xset }^\top)^{-1} - (\tilde{\Xset }\tilde{\Xset }^\top)^{-1} -  (\tilde{\Xset }'\tilde{\Xset }'^\top)^{-1} ( - \Xset ' E^\top - E \Xset ^\top) (\tilde{\Xset }\tilde{\Xset }^\top)^{-1}  \big) \xi \\
\displaystyle  =  \xi^\top  (\tilde{\Xset }'\tilde{\Xset }'^\top)^{-1} (  \Xset ' E^\top + E \Xset ^\top) (\tilde{\Xset }\tilde{\Xset }^\top)^{-1}  \xi \\
\displaystyle  =  \xi^\top  (\tilde{\Xset }'\tilde{\Xset }'^\top)^{-1} (  \Xset ' E^\top - E E^\top + E E^\top + E \Xset ^\top) (\tilde{\Xset }\tilde{\Xset }^\top)^{-1}  \xi \\
\displaystyle  =  \xi^\top  (\tilde{\Xset }'\tilde{\Xset }'^\top)^{-1} (  (\Xset ' - E) E^\top + E (\Xset ^\top + E^\top)) (\tilde{\Xset }\tilde{\Xset }^\top)^{-1}  \xi \\
\displaystyle  =  \xi^\top  (\tilde{\Xset }'\tilde{\Xset }'^\top)^{-1} (  \Xset ' -  E)v \cdot e_{i^*}^\top  (\tilde{\Xset }\tilde{\Xset }^\top)^{-1}  \xi \\
\displaystyle  +  \xi^\top  (\tilde{\Xset }'\tilde{\Xset }'^\top)^{-1} e_{i^*} \cdot v^\top(  \Xset ^\top +E^\top) (\tilde{\Xset }\tilde{\Xset }^\top)^{-1}  \xi \\
\end{array}
\end{equation}
where the second equality follows from~\eqref{eq:a-b-comp} and the last equality follows from $E = e_{i^*} v^\top$. The expression in the last line of~\eqref{eq:stage2-changed} is very similar to the one in the derivation of equation  in~\citet{blocki2012} (p. $17$, second equation array from the top). The difference is that in order for the proof to go trhough, we need to multiply $ (\tilde{\Xset }'\tilde{\Xset }'^\top)^{-1}$ by $\tilde{\Xset }\tilde{\Xset }^\top (\tilde{\Xset }\tilde{\Xset }^\top)^{-1}$ in the second line of~\eqref{eq:stage2-changed}, while the original proof of~\citet{blocki2012} multiplies  $ (\tilde{\Xset }^\top \tilde{\Xset })^{-1}$ by $\tilde{\Xset }'^\top \tilde{\Xset }' (\tilde{\Xset }'^\top \tilde{\Xset }')^{-1}$ (in our notations),  see second equation array from the top, p. $17$ of~\citet{blocki2012}. The remainder of the proof now follows the proof of Theorem $4.1$ of~\citet{blocki2012} (p. $17$).

\subsection{Proof of Theorem~\ref{lemma:dist}}
\label{sef:lemma_dist_proof}

\begin{proof}
	Fix $x, x' \in \Xset$ and their images $z, z' \in \Zset$. 
	If $\sigma_{min}(\Xset) \ge \omega$, according to Algorithm~\ref{alg:jl-dp}, $\Zset =  r^{-1/2} \Xset M$  (line $7$) and 
	\begin{equation*}
	\begin{array}{l}
	\displaystyle  \lVert z - z' \rVert^2  \\
	\displaystyle  =  \lVert r^{-1/2} x M - r^{-1/2}  x' M \rVert^2 \\
	\displaystyle  =  r^{-1} \lVert  x M -   x' M \rVert^2 
	\end{array}
	\end{equation*}
	and Lemma~\ref{lemma:jl} can be immediately applied.
	
	If $\sigma_{min}(\Xset) < \omega$, according to Algorithm~\ref{alg:jl-dp}, $\Zset =  r^{-1/2} \tilde{\Xset} M$  (line $10$) and 
	\begin{equation*}
	\begin{array}{l}
	\displaystyle  \lVert z - z' \rVert^2  \\
	\displaystyle  =  \lVert r^{-1/2}\tilde{x} M - r^{-1/2} \tilde{x}' M \rVert^2 \\
	\displaystyle  =  r^{-1}  \lVert \tilde{x} M - \tilde{x}' M \rVert^2  \\
	\displaystyle  \le   (1 + \nu)  \lVert  \tilde{x}  -  \tilde{x}' \rVert^2  \\
	\displaystyle  \le (1 + \nu)( 1 + \omega^2 /\sigma_{min}^2(\Xset)) \lVert x - x' \rVert^2 
	\end{array}
	\end{equation*}
	where the first inequality follows from Lemma~\ref{lemma:jl} and the second inequality follows from Lemma~\ref{lemma:1}. Similarly, 
	\begin{equation*}
	\begin{array}{l}
	\displaystyle  \lVert z - z' \rVert^2  \\
	\displaystyle  =  \lVert r^{-1/2} \tilde{x} M - r^{-1/2} \tilde{x}' M \rVert^2 \\
	\displaystyle  =  r^{-1} \lVert \tilde{x} M -  \tilde{x}' M \rVert^2\\
	\displaystyle  \ge   (1 - \nu)  \lVert  \tilde{x}  -  \tilde{x}' \rVert^2  \\
	\displaystyle  \ge (1 - \nu)  \lVert x - x' \rVert^2 
	\end{array}
	\end{equation*}
	where the first inequality follows from Lemma~\ref{lemma:jl} and the second inequality follows from Lemma~\ref{lemma:1}.
\end{proof}

\subsection{Bounding the covariance change}

\label{sec:append-bound}
\begin{theorem}
	\label{lemma:k_diff}	
	Let  a dataset  $\Xset  \subset \mathbb{R}^{d}$ 
	be given and  $\sigma_{min}(\Xset) > 0$ be the smallest singular value of $\Xset. $ 
	Let $r\in\mathbb{N}$  be the input parameter  of Algorithm~\ref{alg:jl-dp}, a dataset  $\Zset \subset \mathbb{R}^{r}$ be the output of  Algorithm~\ref{alg:jl-dp}
	and $\omega$  be defined in line $5$  of Algorithm~\ref{alg:jl-dp}.  Let $\reldiam = \text{diam}(\Xset) / l$ where $\text{diam}(\Xset)$ is the diameter of the dataset  $\Xset$.   Let $\nu \in (0, 1/2)$, $\mu \in (0, 1)$ be given. If $\nu \le 2 / \reldiam^2$ and $r \ge 8 \log (n^2 / \mu) / \nu^2$, then the probability of 
	\begin{equation*}
	| k_{z z'} - k_{x x'} | \le C   \cdot k_{x x'}
	\end{equation*}
	for all $x, x' \in \Xset$ and their images under Algorithm~\ref{alg:jl-dp} $z, z' \in \Zset$ is at least $1 - \mu$ where
	\begin{equation}
	\label{eq:C-def}
	\begin{array}{rl}
	C \triangleq & \displaystyle
	\begin{cases}
	\nu \reldiam^2 & \text{if } \sigma_{min}(\Xset) \ge \omega, \\
	\max \Big( \nu \reldiam^2, 1  -  \exp \left(- 0.5 ( \nu  + \nu\omega^2 /\sigma_{min}^2(\Xset) + \omega^2 /\sigma_{min}^2(\Xset)) \reldiam^2 \right) \Big)  & \text{otherwise}.
	\end{cases}
	\end{array}   
	\end{equation}	
	
\end{theorem}
\begin{remark}
	\label{rem:th5}
	It immediately follows from Theorem~\ref{lemma:k_diff} that the probability of $k_{z z'} \le (1 + C)  \cdot k_{x x'}$   for all $x, x' \in \Xset$ and their images $z, z' \in \Zset$  is at least $1 - \mu$.
\end{remark}
\begin{proof}
	\begin{equation*}
	\begin{array}{l}
	\displaystyle k_{z z'} - k_{x x'}\\
	\displaystyle =\sigma_{y}^2 \exp \left(  - 0.5 \lVert z  - z' \rVert^2 / l^2 \right) - \sigma_{y}^2 \exp \left( - 0.5 \lVert x - x' \rVert^2 /l^2 \right) \\ 
	\displaystyle \le \sigma_{y}^2 \exp \left(- 0.5(1 - \nu) \lVert  x - x' \rVert^2 /l^2 \right) - \sigma_{y}^2 \exp \left( - 0.5 \lVert  x - x' \rVert^2 /l^2 \right) \\ 
	\displaystyle = k_{x x'} \bigl(  \exp \left(0.5 \nu \lVert x - x' \rVert^2 /l^2 \right)  - 1 \bigr) \\
	\displaystyle  \le  k_{x x'} \bigl( 2 \cdot  \left(0.5 \nu \lVert x - x' \rVert^2 /l^2 \right)  \bigr)\\
	\displaystyle  \le   k_{x x'} \cdot \nu \reldiam^2
	\end{array}
	\end{equation*}
	where the first inequality follows from Theorem~\ref{lemma:dist} (since the condition $(1 - \nu) \lVert x - x' \rVert^2  \leq \lVert z - z' \rVert^2 $  holds in both cases $\sigma_{min}(\Xset) \ge \omega$ and otherwise),  and the second inequality follows from the identity $\exp c \le 1 + 2 c$ for $c \in (0, 1)$  by setting $c = 0.5 \nu \lVert x - x' \rVert^2 /l^2$ since $\nu \le 2 / \reldiam^2$ and
	\begin{equation}
	\label{eq:lemma_k_diff_1}
	\begin{array}{l}
	\displaystyle 	0.5 \nu \lVert x - x' \rVert^2 /l^2 \\
	\displaystyle 	\le 0.5 \nu \ (\text{diam}(\Xset))^2 /l^2 \\
	\displaystyle 	\le 0.5  \cdot 2 / \reldiam^2 \cdot (\text{diam}(\Xset))^2 /l^2 \\
	\displaystyle = 1. \\
	\end{array}
	\end{equation} 
	
	If $\sigma_{min}(\Xset) \ge \omega$,
	\begin{equation*}
	\begin{array}{l}
	\displaystyle k_{x x'} - k_{z z'} \\
	\displaystyle = \sigma_{y}^2 \exp \left( - 0.5 \lVert x - x' \rVert^2 /l^2 \right) - \sigma_{y}^2 \exp \left(  - 0.5 \lVert z  - z' \rVert^2 / l^2 \right) \\ 
	\displaystyle \le  \sigma_{y}^2 \exp \left( - 0.5 \lVert  x - x' \rVert^2 / l^2 \right) - \sigma_{y}^2 \exp \left(- 0.5(1 + \nu) \lVert  x - x' \rVert^2 /l^2 \right)  \\ 
	\displaystyle = k_{x x'} \bigl(   1  -  \exp \left(- 0.5\nu \lVert x - x' \rVert^2 /l^2 \right) \bigr) \\
	\displaystyle = k_{x x'}  \bigl( \exp \left( 0.5\nu \lVert x - x' \rVert^2 /l^2 \right) - 1 \bigr)
	\exp \left( - 0.5\nu \lVert x - x' \rVert^2 /l^2 \right) \\
	\displaystyle \le k_{x x'} \bigl( \exp \left( 0.5\nu \lVert x - x' \rVert^2 /l^2 \right) - 1 \bigr) \\
	\displaystyle  \le  k_{x x'} \bigl( 2 \cdot  \left(0.5 \nu \lVert x - x' \rVert^2 /l^2 \right)  \bigr)\\
	\displaystyle  \le   k_{x x'} \cdot \nu \reldiam^2 
	\end{array}
	\end{equation*}
	where the first inequality follows from  Theorem~\ref{lemma:dist}, since if $\sigma_{min}(\Xset) \ge \omega$, $C' = 1$ in the statement of Theorem~\ref{lemma:dist}, the second inequality follows from $ 0.5\nu \lVert x - x' \rVert^2 /l^2  \ge 0$ and the third inequality follows from the identity $\exp c \le 1 + 2 c$ for $c \in (0, 1)$  by setting $c = 0.5 \nu \lVert x - x' \rVert^2 /l^2$ and~\eqref{eq:lemma_k_diff_1}.
	
	Similarly, if $\sigma_{min}(\Xset) < \omega$,
	\begin{equation*}
	\begin{array}{l}
	\displaystyle k_{x x'} - k_{z z'} \\
	\displaystyle = \sigma_{y}^2 \exp \left( - 0.5 \lVert x - x' \rVert^2 /l^2 \right) - \sigma_{y}^2 \exp \left(  - 0.5 \lVert z  - z' \rVert^2 / l^2 \right) \\ 
	\displaystyle \le  \sigma_{y}^2 \exp \left( - 0.5 \lVert  x - x' \rVert^2 / l^2 \right) - \sigma_{y}^2 \exp \left(- 0.5(1 + \nu) ( 1 + \omega^2 /\sigma_{min}^2(\Xset))  \lVert  x - x' \rVert^2 /l^2 \right)  \\ 
	\displaystyle = k_{x x'} \bigl(   1  -  \exp \left(- 0.5 ( \nu  + \nu\omega^2 /\sigma_{min}^2(\Xset) + \omega^2 /\sigma_{min}^2(\Xset)) \lVert x - x' \rVert^2 /l^2 \right) \bigr) \\
	\displaystyle \le k_{x x'} \bigl(   1  -  \exp \left(- 0.5 ( \nu  + \nu\omega^2 /\sigma_{min}^2(\Xset) + \omega^2 /\sigma_{min}^2(\Xset)) \reldiam^2 \right) \bigr) \\
	\end{array}
	\end{equation*}
	where the first inequality follows from  Theorem~\ref{lemma:dist}, since if $\sigma_{min}(\Xset) < \omega$, $C' = 1 + \omega^2 /\sigma_{min}^2(\Xset)$ in the statement of Theorem~\ref{lemma:dist}.
	
\end{proof}

\subsection{Proof of Theorem~\ref{th:main}}
\label{sec:main_proof}

First we recall and introduce a few notations which we will use throughout this section. Let  $\Xset  \subset \mathbb{R}^{d}$  be a dataset 
and  its image under Algorithm~\ref{alg:jl-dp} be a dataset  $\Zset \subset \mathbb{R}^{r}$,
$\Zset_{t-1}\triangleq  \{ z_1, \ldots, z_{t-1}  \}$ be a set of transformed inputs selected  by Algorithm~\ref{alg:jl-modeler} run on transformed dataset $\Zset$ after $t-1$ iterations and the preimage of $\Zset_{t-1}$ under Algorithm~\ref{alg:jl-dp} be a set $\Xset_{t-1}\triangleq  \{ x_1, \ldots, x_{t-1}  \}$. Let   $z \in \Zset$  be an (unobserved) transformed input and $x \in \Xset$  be its preimage under Algorithm~\ref{alg:jl-dp}. Let  $f$  be a latent function sampled from a GP. Define
\begin{equation}
\label{eq:main_proof_1}
\begin{array}{l}
\displaystyle \tilde{f}(z) \triangleq f(x) \\
\displaystyle \alpha_t(x, \Xset_{t-1}) \triangleq \mu_{t }(x) + \beta_t^{1/2} \sigma_{t} (x) \\
\displaystyle \alpha_t(z, \Zset_{t-1}) \triangleq \tilde{\mu}_{t}(z) + \beta_t^{1/2} \tilde{\sigma}_{t}(z) \\
\displaystyle z_t \triangleq \argmax_{z \in \Zset} \alpha_t(z, \Zset_{t-1}).
\end{array}
\end{equation}
That is, $\tilde{f}$ is the latent function $f$ defined over the transformed dataset  $\Zset$,  $\alpha_t(z, \Zset_{t-1})$ is the function maximized by Algorithm~\ref{alg:jl-modeler}  at iteration $t$, $\alpha_t(x, \Xset_{t-1})$ is the function maximized by GP-UCB algorithm run on the original dataset, $z_t$ is the transformed input selected by  Algorithm~\ref{alg:jl-modeler} 
at iteration $t$ and $x_t$ is the preimage of $z_t$ under Algorithm~\ref{alg:jl-dp}.

\begin{lemma} 
	\label{lem:srinivas}
	Let   $\delta' \in (0,1)$  be given and $\beta_t \triangleq 2\log( n t^2 \pi^2 / 6\delta')$. Then 
	\begin{equation*}
	|f(x) - \mu_{t} (x)| \leq \beta^{1/2}_t \sigma_{t}(x) \ \ \forall x \in \Xset \ \ \forall t \in\mathbb{N}\ \  
	\end{equation*}
	holds with probability at least $1- \delta'$.
\end{lemma}
\begin{proof}
	Lemma~\ref{lem:srinivas} above corresponds to Lemma $5.1$ in~\citet{srinivas10}; see its proof therein.
\end{proof}

\begin{lemma}
	\label{lem:alpha}
	Let   $\delta' \in (0,1)$  be given and  $\beta_t \triangleq 2\log( n t^2 \pi^2 / 6\delta')$. 	Then the probability of 
	\begin{equation*}
	\tilde{f}(z^*) - \tilde{f}(z_t) \le 2 \max_{x, z} |\alpha_t(z, \Zset_{t-1}) - \alpha_t(x, \Xset_{t-1})| + 2 \beta_t^{1/2} \sigma_{t}(x_t)
	\end{equation*}
	for all $ t \in\mathbb{N}$ is at least $1- \delta'$ where $z^*$ is the maximizer of $\tilde{f}$ and $x \in \Xset$  is the  preimage of $z \in \Zset$ under Algorithm~\ref{alg:jl-dp}.
\end{lemma}
\begin{proof}
	\begin{equation*}
	\begin{array}{l}
	\displaystyle 	\tilde{f}(z^*) - \tilde{f}(z_t)  \\
	\displaystyle	= f(x^*) - f(x_t) \\
	\displaystyle \le \alpha_t(x^*, \Xset_{t-1}) - f(x_t) \\
	\displaystyle = \alpha_t(x^*, \Xset_{t-1}) - \alpha_t(z^*, \Zset_{t-1}) + \alpha_t(z^*, \Zset_{t-1}) - f(x_t) \\
	\displaystyle \le  \alpha_t(x^*, \Xset_{t-1}) - \alpha_t(z^*, \Zset_{t-1}) + \alpha_t(z_t, \Zset_{t-1}) - f(x_t) \\
	\displaystyle =\alpha_t(x^*, \Xset_{t-1}) - \alpha_t(z^*, \Zset_{t-1}) + \alpha_t(z_t, \Zset_{t-1})  - \alpha_t(x_t, \Xset_{t-1})  + \alpha_t(x_t, \Xset_{t-1}) - f(x_t) \\
	\displaystyle \le 2 \max_{x, z}  |\alpha_t(z, \Zset_{t-1}) - \alpha_t(x, \Xset_{t-1})| + \alpha_t(x_t, \Xset_{t-1}) - f(x_t) \\
	\displaystyle \le  2 \max_{x, z}  |\alpha_t(z, \Zset_{t-1}) - \alpha_t(x, \Xset_{t-1})| + 2 \beta_t^{1/2} \sigma_{t}(x_t)
	\end{array}
	\end{equation*}
	where the first equality is due to~\eqref{eq:main_proof_1} and  $x^*$ is the maximizer of $f$, the first and the last inequalities are due to Lemma~\ref{lem:srinivas} and the second inequality is due to the choice of $z_t$ in~\eqref{eq:main_proof_1}.
\end{proof}
Lemma~\ref{lem:alpha} resembles Lemma $5.2$ of~\citet{srinivas10} with an added term $2 \max_{x, z}  |\alpha_t(z, \Zset_{t-1}) - \alpha_t(x, \Xset_{t-1})|$. It suggests that in order to bound regret $\tilde{f}(z^*) - \tilde{f}(z_t)$ incurred by Algorithm~\ref{alg:jl-modeler} at iteration $t$, we need to bound $|\alpha_t(z, \Zset_{t-1}) - \alpha_t(x, \Xset_{t-1})|$. Using the diagonal dominance assumption (Definition~\ref{definition-diag}), we do it in the following two lemmas:

\begin{lemma}
	\label{lem:var}
	Let $C > 0$ be given. If for all $x, x' \in \Xset$ and their images under Algorithm~\ref{alg:jl-dp} $z, z' \in \Zset$ holds $| k_{z z'} - k_{x x'} | \le C   \cdot k_{x x'}$, for  all $t = 1, \ldots, T$ matrix $K_{\Xset_{t-1}\Xset_{t-1}}$ is diagonally dominant, then for  every  unobserved transformed input  $z \in \Zset$  and its preimage under Algorithm~\ref{alg:jl-dp} $x \in \Xset$
	\begin{equation*}
	| \tilde{\sigma}^2_{t}(z) -  \sigma^2_{t}(x) | \le C_1 / \sqrt{|\Xset_{t-1}|} 
	\end{equation*}
	where
	\begin{equation*}
	C_1 \triangleq
	C
	\sigma_y \sqrt{2 \sigma_{y}^2  + \sigma^2_n}
	\Big(  \sqrt{2} ( 1 +C)^2 
	\sigma^2_y  / \sigma_n^2 
	+  (2 + C) C  \Big).
	\end{equation*}
\end{lemma}

\begin{proof}
	\begin{equation}
	\label{eq:17}
	\begin{array}{l}
	\displaystyle  | \tilde{\sigma}^2_{t}(z) -  \sigma^2_{t}(x) | \\ 
	\displaystyle = |   \big( k_{z z} 
	- K_{ z \Zset_{t-1}}
	( K_{\Zset_{t-1}\Zset_{t-1}} + \sigma^2_n I)^{-1}
	K_{\Zset_{t-1}   z  } \big) 
	-   \big( k_{xx} -  K_{  x  \Xset_{t-1}}
	(K_{\Xset_{t-1} \Xset_{t-1}} + \sigma^2_n I)^{-1} 
	K_{\Xset_{t-1}   x  } \big) | \\
	\displaystyle  =  | K_{  z   \Zset_{t-1}}
	( K_{\Zset_{t-1}\Zset_{t-1}} + \sigma^2_n I)^{-1} 
	K_{\Zset_{t-1}   z  }
	- K_{  x   \Xset_{t-1}}
	( K_{\Xset_{t-1}\Xset_{t-1}} + \sigma^2_n I)^{-1}
	K_{\Xset_{t-1}   x  }| \\
	\displaystyle  \le   |  K_{  z   \Zset_{t-1}}
	( K_{\Zset_{t-1} \Zset_{t-1}} + \sigma^2_n  I)^{-1}
	K_{\Zset_{t-1}   z  }
	- K_{  z   \Zset_{t-1}}
	( K_{\Xset_{t-1}\Xset_{t-1}} + \sigma^2_n I)^{-1}
	K_{\Zset_{t-1}   z  } | \\
	\displaystyle  + \,\,   | K_{  z   \Zset_{t-1}}
	( K_{\Xset_{t-1}\Xset_{t-1}} + \sigma^2_n I)^{-1}
	K_{\Zset_{t-1}   z  } 
	- K_{  x   \Xset_{t-1}}
	( K_{\Xset_{t-1}\Xset_{t-1}} + \sigma^2_n I)^{-1}
	K_{\Zset_{t-1}   z  }  | \\
	\displaystyle + \, \,  |  K_{  x   \Xset_{t-1}}
	( K_{\Xset_{t-1}\Xset_{t-1}} + \sigma^2_n  I)^{-1}
	K_{\Zset_{t-1}   z  } 
	- K_{  x   \Xset_{t-1}}
	( K_{\Xset_{t-1}\Xset_{t-1}} + \sigma^2_n  I)^{-1}
	K_{\Xset_{t-1}   x  }| \\
	\displaystyle \le ( 1 + C)^2
	\lVert K_{  x   \Xset_{t-1}} \rVert \cdot 
	\sigma^2_y  / \sigma_n^2 \cdot
	\sqrt{2} C 
	/  \sqrt{|\Xset_{t-1}|}	
	+ (2 + C) C \cdot  \lVert  K_{  x   \Xset_{t-1}} \rVert / \sqrt{|\Xset_{t-1}|} \\
	\displaystyle = C
	\lVert K_{  x   \Xset_{t-1}} \rVert / \sqrt{|\Xset_{t-1}|} 
	\Big(  \sqrt{2} ( 1 +C)^2 
	\sigma^2_y  / \sigma_n^2 
	+  (2 + C) C  \Big)\\
	\displaystyle \le C
	\sigma_y \sqrt{2 \sigma_{y}^2  + \sigma^2_n} / \sqrt{|\Xset_{t-1}|} 
	\Big(  \sqrt{2} ( 1 +C)^2 
	\sigma^2_y  / \sigma_n^2 
	+  (2 + C) C  \Big)\\
	\end{array}
	\end{equation}
	where the first equality is due to \eqref{eq:posterior}, the second equality is due to $k_{xx}=k_{zz}=\sigma_y^2$ for every $x$ and $z$, the first inequality is due to triangle inequality, the second inequality is due to
	\begin{equation*}
	\begin{array}{l}
	\displaystyle  |  K_{  z   \Zset_{t-1}}
	( K_{\Zset_{t-1} \Zset_{t-1}} + \sigma^2_n  I)^{-1}
	K_{\Zset_{t-1}   z  }
	- K_{  z   \Zset_{t-1}}
	(K_{\Xset_{t-1}\Xset_{t-1}} + \sigma^2_n I)^{-1}
	K_{\Zset_{t-1}   z  } | \\
	\displaystyle =  |  K_{  z   \Zset_{t-1}}
	\big( ( K_{\Zset_{t-1} \Zset_{t-1}} + \sigma^2_n  I)^{-1}
	- ( K_{\Xset_{t-1}\Xset_{t-1}} + \sigma^2_n I)^{-1} \big) 
	K_{\Zset_{t-1}   z  } | \\
	\displaystyle \le \lVert K_{  z   \Zset_{t-1}} \rVert^2 \cdot 
	\lVert  ( K_{\Zset_{t-1} \Zset_{t-1}} + \sigma^2_n I)^{-1}	
	- ( K_{\Xset_{t-1}\Xset_{t-1}} + \sigma^2_n I)^{-1} \rVert_2 \\
	\displaystyle \le ( 1 + C)^2 
	\lVert K_{  x   \Xset_{t-1}} \rVert^2 \cdot 
	\lVert  ( K_{\Xset_{t-1}\Xset_{t-1}} + \sigma^2_n  I)^{-1}
	- ( K_{\Xset_{t-1}\Xset_{t-1}} + \sigma^2_n I)^{-1} \rVert_2 \\
	\displaystyle \le ( 1 + C)^2
	\lVert K_{  x   \Xset_{t-1}} \rVert^2 \cdot 
	\lVert (K_{\Zset_{t-1} \Zset_{t-1}} + \sigma^2_n  I)^{-1} (K_{\Zset_{t-1}\Zset_{t-1}} -  K_{\Xset_{t-1}\Xset_{t-1}}) \rVert_2  \cdot
	\lVert  ( K_{\Xset_{t-1}\Xset_{t-1}} + \sigma^2_n  I)^{-1} \rVert_2  \\
	\displaystyle \le ( 1 + C)^2
	\lVert K_{  x   \Xset_{t-1}} \rVert^2 \cdot 
	\lVert (K_{\Zset_{t-1} \Zset_{t-1}} + \sigma^2_n  I)^{-1} \rVert_2 \cdot
	\lVert K_{\Zset_{t-1} \Zset_{t-1}} -  K_{\Xset_{t-1}\Xset_{t-1}} \rVert_2  \cdot
	\lVert  ( K_{\Xset_{t-1}\Xset_{t-1}} + \sigma^2_n  I)^{-1} \rVert_2  \\
	\displaystyle \le ( 1 + C)^2
	\lVert K_{  x   \Xset_{t-1}} \rVert^2 \cdot 
	1 / \sigma_n^2 \cdot
	\lVert K_{\Zset_{t-1} \Zset_{t-1}} -  K_{\Xset_{t-1}\Xset_{t-1}} \rVert_2  \cdot
	\lVert  ( K_{\Xset_{t-1}\Xset_{t-1}} + \sigma^2_n  I)^{-1} \rVert_2  \\
	\displaystyle \le ( 1 + C)^2
	\lVert K_{  x   \Xset_{t-1}} \rVert^2 \cdot 
	1 / \sigma_n^2 \cdot
	\sqrt{2} C \sigma^2_y /  \sqrt{| \Xset_{t-1} | } \cdot
	\lVert  ( K_{\Xset_{t-1}\Xset_{t-1}} + \sigma^2_n  I)^{-1} \rVert_2  \\	
	\displaystyle \le ( 1 + C)^2
	\lVert K_{  x   \Xset_{t-1}} \rVert^2 \cdot 
	1 / \sigma_n^2 \cdot
	\sqrt{2} C \sigma^2_y /  \sqrt{| \Xset_{t-1} | }  \cdot
	1 / (\sqrt{|\Xset_{t-1}|} \lVert K_{x \Xset_{t-1}}\rVert) \\
	\displaystyle = ( 1 + C)^2
	\lVert K_{  x   \Xset_{t-1}} \rVert \cdot 
	\sigma^2_y  / \sigma_n^2 \cdot
	\sqrt{2} C 
	/  |\Xset_{t-1}|			\\
	\displaystyle \le ( 1 + C)^2
	\lVert K_{  x   \Xset_{t-1}} \rVert \cdot 
	\sigma^2_y  / \sigma_n^2 \cdot
	\sqrt{2} C 
	/  \sqrt{|\Xset_{t-1}|}				 
	\end{array}
	\end{equation*}
	where  the first inequality is due to property of quadratic forms $| v^\top A v| \le \lVert v \rVert^2 \cdot  \lVert A \rVert_2$ for any vector $v$ (see Theorem $2.11$, Section II.$2.2$ in~\citet{matrix-pert}), the second inequality follows from the statement of the lemma and Remark~\ref{rem:th5} to Theorem~\ref{lemma:k_diff}, the third inequality follows from Theorem $2.5$ (see Section III.$2.2$ in~\citet{matrix-pert}), the fourth inequality is due to the submultiplicativity of the spectral norm (see Section II.$2.2$, p. $69$ in~\citet{matrix-pert}), the fifth inequality follows from Lemma~\ref{lem:8}, the sixth inequality follows from Lemma~\ref{lem:kxx-kzz},  the second last inequality follows from Lemma~\ref{lem:nghia} and the last inequality follows from $|\Xset_{t-1}|	\ge 1$; 
	
	and 
	\begin{equation*}
	\begin{array}{l}
	\displaystyle   | K_{  z   \Zset_{t-1}}
	( K_{\Xset_{t-1}\Xset_{t-1}} + \sigma^2_n I)^{-1}
	K_{\Zset_{t-1}   z  } 
	- K_{  x   \Xset_{t-1}}
	( K_{\Xset_{t-1}\Xset_{t-1}} + \sigma^2_n I)^{-1}
	K_{\Zset_{t-1}   z  }  | \\
	\displaystyle + \, \,  |  K_{  x   \Xset_{t-1}}
	( K_{\Xset_{t-1}\Xset_{t-1}} + \sigma^2_n  I)^{-1}
	K_{\Zset_{t-1}   z  } 
	- K_{  x   \Xset_{t-1}}
	( K_{\Xset_{t-1}\Xset_{t-1}} + \sigma^2_n  I)^{-1}
	K_{\Xset_{t-1}   x  }| \\
	\displaystyle  =  | (K_{  z   \Zset_{t-1}}
	- K_{  x   \Xset_{t-1}})
	( K_{\Xset_{t-1}\Xset_{t-1}} + \sigma^2_n I)^{-1}
	K_{\Zset_{t-1}   z  }| \\
	\displaystyle + \, \,  |  K_{  x   \Xset_{t-1}}
	( K_{\Xset_{t-1}\Xset_{t-1}} + \sigma^2_n  I)^{-1}
	( K_{\Zset_{t-1}   z  } 
	- K_{\Xset_{t-1}   x  })| \\
	\displaystyle  \le   \lVert K_{  z   \Zset_{t-1}}  - K_{  x   \Xset_{t-1}} \rVert 
	\cdot \lVert( K_{\Xset_{t-1}\Xset_{t-1}} + \sigma^2_n I)^{-1} \rVert_2
	\cdot \lVert K_{\Zset_{t-1}   z  } \rVert\\
	\displaystyle + \, \,   \lVert K_{  x   \Xset_{t-1}} \rVert
	\cdot \lVert ( K_{\Xset_{t-1}\Xset_{t-1}} + \sigma^2_n  I)^{-1} \rVert_2
	\cdot  \lVert K_{\Zset_{t-1}   z  } 	- K_{\Xset_{t-1}   x  } \rVert\\
	\displaystyle \le  (1 + 1 + C) \cdot  \lVert K_{  z   \Zset_{t-1}}-  K_{  x   \Xset_{t-1}} \rVert \cdot
	\lVert  ( K_{\Xset_{t-1}\Xset_{t-1}} + \sigma^2_n I)^{-1} \rVert_2 \cdot
	\lVert K_{\Xset_{t-1}   x  } \rVert \\
	\displaystyle \le (2 + C) \cdot C   \lVert  K_{  x   \Xset_{t-1}} \rVert \cdot
	\lVert  ( K_{\Xset_{t-1}\Xset_{t-1}} + \sigma^2_n  I)^{-1} \rVert_2  
	\cdot  \lVert  K_{  x   \Xset_{t-1}} \rVert\\
	\displaystyle \le  ( 2 + C) \cdot  C   \lVert  K_{  x   \Xset_{t-1}} \rVert \cdot 
	1 / (\sqrt{|\Xset_{t-1}|} \lVert K_{x \Xset_{t-1}}\rVert)
	\cdot  \lVert  K_{  x   \Xset_{t-1}} \rVert\\
	\displaystyle =  (2 + C) C \cdot  \lVert  K_{  x   \Xset_{t-1}} \rVert / \sqrt{|\Xset_{t-1}|} 
	\end{array}
	\end{equation*}
	where the first inequality is due to property of  bilinear forms  $| u^\top A v| \le \lVert u \rVert \cdot  \lVert A \rVert_2 \cdot \lVert v \rVert$ for any vectors $u,v$ (see Theorem $2.11$, Section II.$2.2$ in~\citet{matrix-pert}), the second and the third inequalities follow from  the statement of the lemma and Remark~\ref{rem:th5} to Theorem~\ref{lemma:k_diff} and the last inequality follows from Lemma~\ref{lem:nghia}.
	
	The last inequality  in \eqref{eq:17} follows from
	\begin{equation*}
	\begin{array}{l}
	\displaystyle \lVert K_{  x   \Xset_{t-1}} \rVert^2 \\
	\displaystyle = \lVert K_{  x   \Xset_{t-1}} \rVert^2 
	\cdot \psi^{-1}_{max} ( K_{\Xset_{t-1}\Xset_{t-1}} + \sigma^2_n I)
	\cdot  \psi_{max}( K_{\Xset_{t-1}\Xset_{t-1}} + \sigma^2_n  I)\\
	\displaystyle = \lVert K_{  x   \Xset_{t-1}} \rVert^2 
	\cdot \psi_{min}( ( K_{\Xset_{t-1}\Xset_{t-1}} + \sigma^2_n I)^{-1})
	\cdot  \psi_{max}( K_{\Xset_{t-1}\Xset_{t-1}} + \sigma^2_n  I)\\
	\displaystyle = \lVert K_{  x   \Xset_{t-1}} \rVert^2 
	\cdot \psi_{min}( ( K_{\Xset_{t-1}\Xset_{t-1}} + \sigma^2_n I)^{-1})
	\cdot  \lVert K_{\Xset_{t-1}\Xset_{t-1}} + \sigma^2_n I \lVert_2 \\
	\displaystyle = \lVert K_{  x   \Xset_{t-1}} \rVert^2 
	\cdot \psi_{min}( ( K_{\Xset_{t-1}\Xset_{t-1}} + \sigma^2_n  I)^{-1})
	\cdot  (\lVert K_{\Xset_{t-1}\Xset_{t-1}}  \lVert_2  + \sigma^2_n) \\
	\displaystyle \le \lVert K_{  x   \Xset_{t-1}} \rVert^2 
	\cdot \psi_{min}( ( K_{\Xset_{t-1}\Xset_{t-1}} + \sigma^2_n I)^{-1})
	\cdot (2 \sigma_{y}^2  + \sigma^2_n) \\
	\displaystyle \le K_{  x   \Xset_{t-1}}
	( K_{\Xset_{t-1}\Xset_{t-1}} + \sigma^2_n I)^{-1}
	K_{\Xset_{t-1}   x  } 
	\cdot  (2 \sigma_{y}^2  + \sigma^2_n) \\
	\displaystyle \le k_{xx}
	\cdot   (2 \sigma_{y}^2  + \sigma^2_n) \\		
	\displaystyle =  \sigma^2_y (2 \sigma_{y}^2  + \sigma^2_n) 
	\end{array}
	\end{equation*}
	where $\psi_{max}(\cdot )$ and  $\psi_{min}(\cdot)$ denote the largest and the smallest eigenvalues of a matrix, respectively,  the first fourth equalities are  properties of eigenvalues, the first inequality is due to Lemma~\ref{lem:Kxx}, the second inequality follows from Lemma~\ref{lem:min_eigen}, the third inequality follows from the fact that conditioning does not increase variance and the last equality is due to $k_{xx}=\sigma^2_y$.
\end{proof}

\begin{lemma}
	\label{th:mean}
	Let $C > 0$ be given. If for all $x, x' \in \Xset$ and their images under Algorithm~\ref{alg:jl-dp} $z, z' \in \Zset$ holds $| k_{z z'} - k_{x x'} | \le C   \cdot k_{x x'}$,  for all $t = 1, \ldots, T$ matrix $K_{\Xset_{t-1}\Xset_{t-1}}$ is diagonally dominant and $| y_t | \le L$, then for  every   unobserved transformed input $z \in \Zset$  and its preimage under Algorithm~\ref{alg:jl-dp}  $x \in \Xset$
	\begin{equation*}
	| \tilde{\mu}_{t}(z) - \mu_{t}(x) | \le CL + C_2   / \sqrt{| \Xset_{t-1} | } 
	\end{equation*}
	where 
	\begin{equation*}
	C_2 =  \sqrt{2} (1 + C )	\cdot C \sigma^2_y / \sigma^2_n \cdot L.
	\end{equation*}
\end{lemma}

\begin{proof}
	\begin{equation*}
	\begin{array}{l}
	\displaystyle | \tilde{\mu}_{t} (z) - \mu_{t }(x) |  \\
	\displaystyle = |K_{z  \Zset_{t-1}}
	(K_{\Zset_{t-1} \Zset_{t-1}} + \sigma^2_n   I)^{-1}
	\vec{y}_{t-1} 
	- K_{x  \Xset_{t-1}}
	(K_{\Xset_{t-1} \Xset_{t-1}} + \sigma^2_n  I)^{-1}
	\vec{y}_{t-1} | \\
	\displaystyle \le |K_{z \Zset_{t-1}} 
	(K_{\Xset_{t-1} \Xset_{t-1}} + \sigma^2_n   I)^{-1}
	\vec{y}_{t-1}
	-  K_{x  \Xset_{t-1}}
	(K_{\Xset_{t-1} \Xset_{t-1}} + \sigma^2_n   I)^{-1}
	\vec{y}_{t-1} | \\
	\displaystyle +  \,\,  | K_{z  \Zset_{t-1}} 
	(K_{\Zset_{t-1} \Zset_{t-1}} + \sigma^2_n  I)^{-1}
	\vec{y}_{t-1}
	- K_{z  \Zset_{t-1}} 
	(K_{\Xset_{t-1} \Xset_{t-1}} + \sigma^2_n  I)^{-1} 
	\vec{y}_{t-1}| \\
	\displaystyle = |(K_{z \Zset_{t-1}} -  K_{x  \Xset_{t-1}})
	(K_{\Xset_{t-1} \Xset_{t-1}} + \sigma^2_n   I)^{-1}
	\vec{y}_{t-1}| \\
	\displaystyle +  \,\,  | K_{z  \Zset_{t-1}} 
	\big((K_{\Zset_{t-1} \Zset_{t-1}} + \sigma^2_n  I)^{-1}
	-  (K_{\Xset_{t-1} \Xset_{t-1}} + \sigma^2_n  I)^{-1} \big)
	\vec{y}_{t-1}| \\
	\displaystyle \le C  \cdot L
	+ C_2 /   \sqrt{| \Xset_{t-1} | } 
	\end{array}
	\end{equation*}
	where the first equality is due to \eqref{eq:posterior}, the first inequality is due to triangle inequality and the second inequality  follows from
	\begin{equation*}
	\begin{array}{l}
	\displaystyle |(K_{z \Zset_{t-1}} -  K_{x  \Xset_{t-1}})
	(K_{\Xset_{t-1} \Xset_{t-1}} + \sigma^2_n   I)^{-1}
	\vec{y}_{t-1}| \\
	\displaystyle  \le \lVert K_{z \Zset_{t-1}} -  K_{x  \Xset_{t-1}} \rVert 
	\cdot  \lVert (K_{\Xset_{t-1} \Xset_{t-1}} + \sigma^2_n   I)^{-1} \rVert_2
	\cdot \lVert  \vec{y}_{t-1} \rVert \\
	\displaystyle  \le C  \lVert  K_{  x   \Xset_{t-1}} \rVert 
	\cdot  \lVert (K_{\Xset_{t-1} \Xset_{t-1}} + \sigma^2_n   I)^{-1} \rVert_2
	\cdot \lVert  \vec{y}_{t-1} \rVert \\
	\displaystyle  \le  C \lVert  K_{  x   \Xset_{t-1}} \rVert 
	\cdot  1 / (\sqrt{|\Xset_{t-1}|} \lVert K_{x \Xset_{t-1}}\rVert)
	\cdot \lVert  \vec{y}_{t-1}\rVert \\
	\displaystyle  \le C \cdot L
	\end{array}
	\end{equation*}
	where the first inequality is due to property of bilinear forms  $| u^\top A v | \le \lVert u \rVert \cdot  \lVert A \rVert_2 \cdot \lVert v \rVert$ for any vectors $u,v$ (see Theorem $2.11$, Section II.$2.2$ in~\citet{matrix-pert}), the second inequality follows  from  the statement of the lemma, the third inequality follows from Lemma~\ref{lem:nghia} and  the last inequality follows from the condition $| y_t | \le L$ for  all $t = 1, \ldots, T$;

	and 
	\begin{equation*}
	\begin{array}{l}
	\displaystyle | K_{z  \Zset_{t-1}} 
	\big((K_{\Zset_{t-1} \Zset_{t-1}} + \sigma^2_n   I)^{-1}
	-  (K_{\Xset_{t-1} \Xset_{t-1}} + \sigma^2_n   I)^{-1} \big)
	\vec{y}_{t-1} | \\
	\displaystyle  \le  \lVert K_{z  \Zset_{t-1}} \rVert 
	\cdot \lVert (K_{\Zset_{t-1} \Zset_{t-1}} + \sigma^2_n   I)^{-1}
	-  (K_{\Xset_{t-1} \Xset_{t-1}} + \sigma^2_n  I)^{-1} \rVert_2
	\cdot \lVert  \vec{y}_{t-1} \rVert \\
	\displaystyle   \le \lVert K_{z  \Zset_{t-1}} \rVert 
	\cdot  \lVert (K_{\Zset_{t-1} \Zset_{t-1}} + \sigma^2_n  I)^{-1} \rVert_2 
	\cdot \lVert (K_{\Zset_{t-1} \Zset_{t-1}} -  K_{\Xset_{t-1} \Xset_{t-1}})
	( K_{\Xset_{t-1} \Xset_{t-1}} + \sigma^2_n I)^{-1} \rVert_2 
	\cdot \lVert  \vec{y}_{t-1} \rVert \\
	\displaystyle   \le \lVert K_{z  \Zset_{t-1}} \rVert 
	\cdot  \lVert (K_{\Zset_{t-1} \Zset_{t-1}} + \sigma^2_n  I)^{-1} \rVert_2 
	\cdot \lVert K_{\Zset_{t-1} \Zset_{t-1}} -  K_{\Xset_{t-1} \Xset_{t-1}} \rVert_2  \cdot \lVert  ( K_{\Xset_{t-1} \Xset_{t-1}} + \sigma^2_n I)^{-1} \rVert_2 
	\cdot \lVert  \vec{y}_{t-1} \rVert \\
	\displaystyle   \le \lVert K_{z  \Zset_{t-1}} \rVert 
	\cdot  1 / \sigma^2_n
	\cdot \lVert K_{\Zset_{t-1} \Zset_{t-1}} -  K_{\Xset_{t-1} \Xset_{t-1}} \rVert_2 
	\cdot \lVert ( K_{\Xset_{t-1} \Xset_{t-1}} + \sigma^2_n I)^{-1} \rVert_2 
	\cdot \lVert  \vec{y}_{t-1} \rVert \\
	\displaystyle   \le \lVert K_{z  \Zset_{t-1}} \rVert 
	\cdot  1 / \sigma^2_n
	\cdot  \sqrt{2} C \sigma^2_y /  \sqrt{| \Xset_{t-1} | }  
	\cdot \lVert ( K_{\Xset_{t-1} \Xset_{t-1}} + \sigma^2_n I)^{-1} \rVert_2 
	\cdot \lVert  \vec{y}_{t-1} \rVert \\
	\displaystyle   \le \lVert K_{z  \Zset_{t-1}} \rVert 
	\cdot  1 / \sigma^2_n
	\cdot  \sqrt{2} C \sigma^2_y /  \sqrt{| \Xset_{t-1} | }  
	\cdot 1 / (\sqrt{|\Xset_{t-1}|} \lVert K_{x \Xset_{t-1}}\rVert)
	\cdot \lVert  \vec{y}_{t-1} \rVert \\
	\displaystyle   \le (1 + C ) \lVert K_{x \Xset_{t-1}}\rVert
	\cdot  1 / \sigma^2_n
	\cdot  \sqrt{2} C \sigma^2_y / \sqrt{| \Xset_{t-1} | } 
	\cdot 1 / (\sqrt{|\Xset_{t-1}|} \lVert K_{x \Xset_{t-1}}\rVert)
	\cdot \lVert  \vec{y}_{t-1} \rVert \\
	\displaystyle   \le \sqrt{2} (1 + C )
	\cdot C \sigma^2_y / \sigma^2_n \cdot L /   \sqrt{| \Xset_{t-1} | }  \\
	\displaystyle   = C_2 /   \sqrt{| \Xset_{t-1} | }  \\
	\end{array}
	\end{equation*}
	where the first inequality is due to property of bilinear forms  $| u^\top A v | \le \lVert u \rVert \cdot  \lVert A \rVert_2 \cdot \lVert v \rVert$ for any vectors $u,v$ (see Theorem $2.11$, Section II.$2.2$ in~\citet{matrix-pert}), the second inequality follows from Theorem $2.5$ (see Section III.$2.2$ in~\citet{matrix-pert}), the third inequality is due to the submultiplicativity of the spectral norm (see Section II.$2.2$, p. $69$ in~\citet{matrix-pert}) the fourth inequality follows from Lemma~\ref{lem:8}, the fifth inequality follows from  Lemma~\ref{lem:kxx-kzz}, the third last inequality follows from Lemma~\ref{lem:nghia}, the second last inequality follows from  the statement of the lemma and Remark~\ref{rem:th5} to Theorem~\ref{lemma:k_diff} and the  last inequality follows from the condition  $| y_t | \le L$ for  all $t = 1, \ldots, T$.
\end{proof}

\noindent \textit{Proof of the theorem.}  By Lemma~\ref{lem:alpha} for $\delta' = \delta_{ucb} / 2$ and $\beta_t = 2\log( n t^2 \pi^2 / 3\delta_{ucb})$ for all $t \in \mathbb{N}$:
\begin{equation}
\label{eq:th3-1}
\begin{array}{l}
\displaystyle r_t \\
\displaystyle =	f(x^*) - f(x_t) \\
\displaystyle =	\tilde{f}(z^*) - \tilde{f}(z_t) \\
\displaystyle \le 2 \max_{x,z} |\alpha_t(z, \Zset_{t-1}) - \alpha_t(x, \Xset_{t-1})| + 2 \beta_t^{1/2} 	\sigma_{t}(x_t)\\
\displaystyle \le 2 \max_{x,z} | \tilde{\mu}_{t}(z) - \mu_{t}(x) |
+ 2 \beta_t^{1/2} \max_{x,z} | \tilde{\sigma}^2_{t}(z) -  \sigma^2_{t}(x) |
+ 2 \beta_t^{1/2} \sigma_{t }(x_t) \\
\end{array}
\end{equation}
with probability at least $1- \delta_{ucb} / 2$ where the second equality follows from~\eqref{eq:main_proof_1}, the first inequality follows from Lemma~\ref{lem:alpha} and the second inequality follows from triangle inequality. Suppose $\nu \in (0, \min(1/2, 2 / \reldiam^2))$, $\mu \in (0, 1)$  are given (we will set the exact values of $\mu, \nu$ later) and the input parameter of Algorithm~\ref{alg:jl-dp} $r \ge 8 \log (n^2 / \mu) / \nu^2$. 
By Theorem~\ref{lemma:k_diff} for all $x, x' \in \Xset$ and their images under Algorithm~\ref{alg:jl-dp} $z, z' \in \Zset$ holds $| k_{z z'} - k_{x x'} | \le C   \cdot k_{x x'}$ with probability at least $1 - \mu$. Let  $\mu = \delta_{ucb} / 2$. Then we can apply  Lemma~\ref{lem:var} and Lemma~\ref{th:mean} to~\eqref{eq:th3-1}.  Using  the union bound we obtain that for all $t = 1, \ldots, T$ 
\begin{equation}
\label{eq:th3-1a}
\begin{array}{l}
\displaystyle r_t \\
\displaystyle \le 2 \max_{x,z} | \tilde{\mu}_{t } (z) - \mu_{t}(x) |
+ 2 \beta_t^{1/2} \max_{x,z} | \tilde{\sigma}^2_{t}(z) -  \sigma^2_{t}(x) |
+ 2 \beta_t^{1/2} \sigma_{t}(x_t) \\
\displaystyle \le 2 (CL + C_2  / \sqrt{| \Xset_{t-1} | }) 
+ 2 C_1 \beta^{1/2}_t   / \sqrt{| \Xset_{t-1} | }  
+ 2 \beta_t^{1/2} \sigma_{t}(x_t) \\
\end{array}
\end{equation}
with probability at least $1 - \delta_{ucb}$ where $C_1$ and $C_2$ are defined in Lemma~\ref{lem:var} and Lemma~\ref{th:mean}, respectively. Summing over $t=1, \ldots, T$:
\begin{equation}
\label{eq:th3-2}
\begin{array}{l}
\displaystyle \sum_{t=1}^T r_t^2 \\
\displaystyle  \le 4 \sum_{t=1}^T \big( 
CL + C_2  / \sqrt{| \Xset_{t-1} | }
+ C_1 \beta^{1/2}_t   / \sqrt{| \Xset_{t-1} | }  
+  \beta_t^{1/2} \sigma_{t}(x_t) \big)^2 \\
\displaystyle \le 12 \sum_{t=1}^T \big( C^2 L^2
+  (C_2  + C_1 \beta^{1/2}_t)^2 / |\Xset_{t-1}|
+  \beta_t \sigma^2_{t }(x_t) \big)\\
\displaystyle =  12 C^2 L^2 T
+ 12  \sum_{t=1}^T (C_2  + C_1 \beta^{1/2}_t)^2  |\Xset_{t-1}|
+ 12 \sum_{t=1}^T \beta_t \sigma^2_{t}(x_t) \\
\displaystyle \le  12 C^2 L^2 T
+ 24  (C_2  + C_1 \beta^{1/2}_T)^2 \log T
+ 12 \beta_T \sum_{t=1}^T  \sigma^2_{t}(x_t) \\
\displaystyle \le 12 C^2 L^2 T
+ 24 (C_2  + C_1 \beta^{1/2}_T)^2 \log T
+ 12 \beta_T / \log (1 + \sigma_n^{-2})
\sum_{t=1}^T \log (1 + \sigma_n^{-2} \sigma^2_{t}(x_t))\\
\displaystyle \le 12 C^2 L^2 T
+ 24 (C_2  + C_1 \beta^{1/2}_T)^2\log T
+ 24 \beta_T / \log (1 + \sigma_n^{-2}) \cdot \gamma_T			
\end{array}
\end{equation}
where the first inequality follows from \eqref{eq:th3-1a}, the second inequality follows from identity $(a + b + c)^2 \le 3(a^2 + b^2 + c^2)$, the third inequality follows from $\sum_{t=1}^T 1/|\Xset_{t-1}| \le \sum_{t=1}^T 1/t \le 2 \log T$ and the fact that $\beta_t$ is nondecreasing, the fourth inequality corresponds to an intermediate step of Lemma $5.4$ in~\citet{srinivas10} and the last step follows from Lemma $5.3$ and Lemma $5.4$ in~\citet{srinivas10} where $\gamma_T \triangleq \max_{\Xset_{T} \subset \Xset} \mathbb{I} [\vec{f}_{\Xset}; \vec{y}_{t-1}] = \mathcal{O}\big( (\log T)^{d+1} \big)$ and $\mathbf{f}_{\Xset} \triangleq (f(x))^\top_{x \in \Xset}$ (see Theorem~$5$ in~\citet{srinivas10}). Therefore, 
\begin{equation}
\label{eq:regret_final}
\begin{array}{l}
\displaystyle S_T^2  \\
\displaystyle \le  R_T^2 / T^2 \\
\displaystyle \le \sum_{t=1}^T r_t^2 / T\\
\displaystyle \le 12 C^2 L^2 
+ 24 (C_2  + C_1 \beta^{1/2}_T)^2\log T / T
+ 24 \beta_T / \log (1 + \sigma_n^{-2}) \gamma_T	/ T
\end{array}
\end{equation}
where the second inequality follows from Cauchy-Schwarz inequality and the last inequality follows from \eqref{eq:th3-2}. 
If $\sigma_{min}(\Xset) \ge \omega$ then, according to Theorem~\ref{lemma:k_diff}, $C = \nu \reldiam^2$.
To guarantee that $12 C^2 L^2 \le \epsilon^2_{ucb}$ and to satisfy the premise of Lemma~\ref{lemma:jl} (i.e. $\nu \le 1/2$)  and  Theorem~\ref{lemma:k_diff} (i.e. $\nu \le 2 / \reldiam^2$), we need to set the value of $\nu = \min(\varepsilon_{ucb}/ (2 \sqrt{3} \reldiam^2 L), 2 / \reldiam^2, 1/2)$. 

Since $\nu \le 2 / \reldiam^2$ and hence $C = \nu \reldiam^2 \le 2$ 
\begin{equation*}
\begin{array}{l}
\displaystyle C_1 \\
\displaystyle = C
\sigma_y \sqrt{2 \sigma_{y}^2  + \sigma^2_n}
\Big(  \sqrt{2} ( 1 +C)^2 
\sigma^2_y  / \sigma_n^2
+  (2 + C) C  \Big) \\
\displaystyle \le 2
\sigma_y \sqrt{2 \sigma_{y}^2  + \sigma^2_n}
\Big(  \sqrt{2} ( 1 +2)^2 
\sigma^2_y  / \sigma_n^2 
+  (2 + 2) \cdot 2  \Big) \\
\displaystyle = \mathcal{O}	\Big( 
\sigma_y \sqrt{\sigma_{y}^2  + \sigma^2_n}
(\sigma^2_y  / \sigma_n^2 + 1  ) \Big) \\
\end{array}
\end{equation*}
and
\begin{equation*}
\begin{array}{l}
\displaystyle C_2 \\
\displaystyle = \sqrt{2} (1 + C )	\cdot C \sigma^2_y / \sigma^2_n \cdot L \\
\displaystyle \le \sqrt{2} (1 + 2 )	\cdot 2 \sigma^2_y / \sigma^2_n \cdot L \\
\displaystyle = \mathcal{O}(\sigma^2_y / \sigma^2_n \cdot L)
\end{array}
\end{equation*}
where $C_1$ and $C_2$ are defined in Lemma~\ref{lem:var} and Lemma~\ref{th:mean}, respectively.

\begin{remark}
	\label{rem:2_ext}
	If $\sigma_{min}(\Xset) < \omega$, a similar form of regret bound to that of~\eqref{eq:regret_final} can be proven: According to Theorem~\ref{lemma:k_diff}, $C=\max( \nu \reldiam^2, 1  -  \exp \left(- 0.5 ( \nu  + \nu\omega^2 /\sigma_{min}^2(\Xset) + \omega^2 /\sigma_{min}^2(\Xset)) \reldiam^2 \right) ) $ instead of $C = \nu \reldiam^2$ and the entire proof of Theorem~\ref{th:main} can be directly copied to reach~\eqref{eq:regret_final}. In this case, however, the term $12 C^2 L^2$ in~\eqref{eq:regret_final} cannot be  set arbitrarily small. That is explained by the fact that when $\sigma_{min}(\Xset) < \omega$, Algorithm~\ref{alg:jl-dp} increases the singular values of dataset  $\Xset$ (see line $9$) and the pairwise distances between the original inputs from $\Xset$ are no longer approximately the same as the  distances between their respective transformed images (see Theorem~\ref{lemma:dist}) resulting in a looser regret bound.
\end{remark}

\subsection{Auxiliary results}

\begin{lemma}
	\label{lemma:1}
	
	Let  a dataset  $\Xset  \subset \mathbb{R}^{d}$ be given. Let a dataset  $\tilde{\Xset } \subset \mathbb{R}^{d}$ be defined in line $9$ of Algorithm~\ref{alg:jl-dp}
	(i.e., $\tilde{\Xset } = U \sqrt{\Sigma^2  + \omega^2 I_{n \times d}} V^\top$ where  $\Xset  = U\Sigma V^\top$ is the  singular value decomposition of $\Xset $).
	Let $\sigma_{min}(\Xset) > 0$ be the smallest singular value of $\Xset $. Then for all $x, x' \in \Xset$ and their corresponding  $\tilde{x}, \tilde{x}' \in \tilde{\Xset }$ (when viewing datasets  $\Xset$ and $\tilde{\Xset }$ as matrices)
	
	\begin{equation*}
	\lVert x - x' \rVert 
	\le  \lVert \tilde{x} - \tilde{x}' \rVert 
	\le \sqrt{ 1 + \omega^2 /\sigma_{min}^2(\Xset) } \lVert x - x' \rVert.
	\end{equation*}
\end{lemma}
\begin{proof}
	Denote the rows of $U$ as $u_{(i)}$ so that 
	\begin{equation*}
	U = 
	\begin{bmatrix}
	u_{(1)} \\
	\vdots \\
	u_{(n)}
	\end{bmatrix}.
	\end{equation*}
	For $i = 1, \ldots, n$ denote the input in the $i$-th row of the datset $\Xset$ ($\tilde{\Xset}$) viewed as matrix as $x_{(i)}$ ($\tilde{x}_{(i)}$). From the singular value decomposition, $x_{(i)} = u_{(i)} \Sigma V^\top$ and $\tilde{x}_{(i)} = u_{(i)}  \sqrt{\Sigma^2 + I_{n \times d} \omega^2} V^\top$ 
	Then for $i, j = 1, \ldots, n$
	\begin{equation}
	\label{eq:lem1-1}
	\begin{array}{l}
	\displaystyle  \lVert \tilde{x}_{(i)} - \tilde{x}_{(j)} \rVert^2   \\
	\displaystyle = \lVert (u_{(i)} - u_{(j)}) \sqrt{\Sigma^2 + \omega^2 I_{n \times d}  } V^\top \rVert^2  \\
	\displaystyle = (u_{(i)} - u_{(j)}) \sqrt{\Sigma^2 +\omega^2 I_{n \times d} } V^\top V \sqrt{\Sigma^2 +  \omega^2 I_{n \times d} }^\top (u_{(i)} - u_{(j)})^\top \\
	\displaystyle =  (u_{(i)} - u_{(j)}) \sqrt{\Sigma^2 + \omega^2 I_{n \times d} } \sqrt{\Sigma^2 +  \omega^2 I_{n \times d} }^\top (u_{(i)} - u_{(j)})^\top \\
	\displaystyle =  \sum_{k=1}^{\min(n, d)}(u_{(i) k} - u_{(j)k})^2  (\sigma_k^2 + \omega^2)  \\
	\displaystyle \le 
	\sum_{k=1}^{\min(n, d)}(u_{(i) k} - u_{(j)k})^2   \sigma_k^2 	 ( 1 + \omega^2 /\sigma_{min}^2(\Xset) ) \\
	\displaystyle =   ( 1 + \omega^2 /\sigma_{min}^2(\Xset) )   (u_{(i)} - u_{(j)}) \Sigma \Sigma^\top (u_{(i)} - u_{(j)})^\top \\
	\displaystyle =  ( 1 + \omega^2 /\sigma_{min}^2(\Xset) )   (u_{(i)} - u_{(j)}) \Sigma V^\top V \Sigma^\top (u_{(i)} - u_{(j)})^\top \\
	\displaystyle = ( 1 + \omega^2 /\sigma_{min}^2(\Xset) )  \lVert (u_{(i)} - u_{(j)}) \Sigma V^\top \rVert^2  \\
	\displaystyle  = ( 1 + \omega^2 /\sigma_{min}^2(\Xset) )   \lVert x_{(i)}  - x_{(j)}  \rVert^2   \\
	\end{array}
	\end{equation}
	
where the second and the second last equalities follow from $\lVert v \rVert^2 = v v\top$ for any row vector $v$, the third  and the third last equalities follow from orthonormality of matrix $V$, and the  inequality follows from
	\begin{equation*}
	\begin{array}{l}
	\displaystyle \sigma_k^2 + \omega^2\\
	\displaystyle = \sigma_k^2 ( 1 + \omega^2 /\sigma_k^2 )  \\
	\displaystyle \le \sigma_k^2 ( 1 + \omega^2 /\sigma_{min}^2(\Xset) ) 
	\end{array}
	\end{equation*}
	where the inequality follows from $\sigma_k \ge \sigma_{min}(\Xset)$ for every  $k= 1, \ldots, \min(n, d)$.

	Similarly,
	\begin{equation}
	\label{eq:lem1-2}
	\begin{array}{l}
	\displaystyle  \lVert \tilde{x}_{(i)}  - \tilde{x}_{(j)} \rVert^2   \\
	\displaystyle  = \sum_{k=1}^{\min(n, d)}(u_{{(i)}k} - u_{{(j)} k})^2  (\sigma_k^2 + \omega^2) \\
	\displaystyle =  \sum_{k=1}^{\min(n, d)}(u_{{(i)} k} - u_{ {(j)}k})^2  \sigma_k^2 + \omega^2 \sum_{k=1}^{\min(n, d)}(u_{ {(i)} k} - u_{{(j)} k})^2 \\ 
	\displaystyle \ge \sum_{k=1}^{\min(n, d)}(u_{ {(i)} k} - u_{ {(j)} k})^2  \sigma_k^2 \\
	\displaystyle =  \lVert x_{(i)}  - x_{(j)}  \rVert^2  
	\end{array}
	\end{equation}
	where the first and the last equalities follow from the fourth  and  the fifth equalities of~\eqref{eq:lem1-1}, respectively. Since \eqref{eq:lem1-1} and \eqref{eq:lem1-2} both hold for all $i, j = 1, \ldots, n$, the lemma follows.
\end{proof} 

\begin{lemma}
	\label{lem:8}
	In the notations of Section~\ref{sec:main_proof}, for all $t = 1, \ldots, T$ holds $\lVert  ( K_{\Zset_{t-1} \Zset_{t-1}} + \sigma^2_n I)^{-1} \rVert_2 \le  1 / \sigma^2_{n}$.  
\end{lemma}
\begin{proof}
	Since $(K_{{\Zset}_{t-1}{\Zset}_{t-1}} + \sigma^2_n  I)^{-1} $ is positive definite, by  definition of spectral norm for all $t = 1, \ldots, T$ and  $\Zset_{t-1}$
	\begin{equation*}
	\begin{array}{l}
	\displaystyle \lVert  ( K_{{\Zset}_{t-1}{\Zset}_{t-1}} + \sigma^2_n  I)^{-1} \rVert_2  \\
	\displaystyle = \psi_{max}( ( K_{{\Zset}_{t-1}{\Zset}_{t-1}} + \sigma^2_n  I)^{-1}) \\
	\displaystyle =	\frac{1} {\psi_{min}(K_{{\Zset}_{t-1}{\Zset}_{t-1}} + \sigma^2_n I)} \\		
	\displaystyle = \frac{1} {\psi_{min}(K_{{\Zset}_{t-1}{\Zset}_{t-1}}) + \sigma^2_n}\\
	\displaystyle \le 1 / \sigma_{n}^2 \\											
	\end{array}
	\end{equation*}
	where $\psi_{max}(\cdot )$ and  $\psi_{min}(\cdot)$ denote the largest and the smallest eigenvalues of a matrix, respectively, the second and the third equalities are properties of eigenvalues and the inequality is due to the fact that matrix $K_{{\Zset}_{t-1}{\Zset}_{t-1}}$ is positive semidefinite.
\end{proof}

\begin{lemma}
	\label{lem:kxx-kzz}
	In the notations of Section~\ref{sec:main_proof}, 	if for all $x, x' \in \Xset$ and their images under Algorithm~\ref{alg:jl-dp} $z, z' \in \Zset$ holds $| k_{z z'} - k_{x x'} | \le C   \cdot k_{x x'}$, 
	and for all $t = 1, \ldots, T$  matrix $K_{\Xset_{t-1} \Xset_{t-1}}$ is diagonally dominant (Definition~\ref{definition-diag}), then 
	\begin{equation*}
	\lVert K_{\Zset_{t-1} \Zset_{t-1}} - K_{\Xset_{t-1} \Xset_{t-1}} \rVert_2 \le \sqrt{2} C \sigma^2_y /  \sqrt{| \Xset_{t-1} | }.
	\end{equation*}
\end{lemma}
\begin{proof}
	Fix $t = 1, \ldots, T$. For some $i = 1, \ldots, t - 1$:
	\begin{equation*}
	\begin{array}{l}
	\displaystyle \lVert K_{\Zset_{t-1} \Zset_{t-1}} - K_{\Xset_{t-1} \Xset_{t-1}} \rVert^2_2 \\
	\displaystyle = \psi_{max} \big( (K_{\Zset_{t-1} \Zset_{t-1}} - K_{\Xset_{t-1} \Xset_{t-1}} )^\top 
	( K_{\Zset_{t-1} \Zset_{t-1}} - K_{\Xset_{t-1} \Xset_{t-1}} ) \big)  \\
	\displaystyle = \psi_{max} \big( ( K_{\Zset_{t-1} \Zset_{t-1}} - K_{\Xset_{t-1} \Xset_{t-1}} )^2 \big) \\
	\displaystyle \le \sum_{j, j \neq i} |[( K_{\Zset_{t-1} \Zset_{t-1}} - K_{\Xset_{t-1} \Xset_{t-1}} )^2 ]_{ij}| + 
	[( K_{\Zset_{t-1} \Zset_{t-1}} - K_{\Xset_{t-1} \Xset_{t-1}} )^2 ]_{ii} \\
	\displaystyle  \le 2 C^2  \sigma_y^4 / \big( \sqrt{| \Xset_{t-1} |  -1} + 1 \big)^2 \\
	\displaystyle  \le 2 C^2  \sigma_y^4 / | \Xset_{t-1} | 
	\end{array}
	\end{equation*}
	where $\psi_{max}(\cdot )$  denotes the largest eigenvalue of a matrix, the first equality is the definition of spectral norm, the second equality follows from the fact that matrices $ K_{\Zset_{t-1} \Zset_{t-1}}$ and $K_{\Xset_{t-1} \Xset_{t-1}}$ are symmetric, the first inequality is due to Gershgorin circle theorem, the last inequality follows from $ \sqrt{| \Xset_{t-1} |  -1} + 1 \ge  \sqrt{| \Xset_{t-1} |}$ and the second last inequality follows from 
	\begin{equation*}
	\begin{array}{l}
	\displaystyle \sum_{j, j \neq i} |[( K_{\Zset_{t-1} \Zset_{t-1}} - K_{\Xset_{t-1} \Xset_{t-1}} )^2 ]_{ij}| \\
	\displaystyle  =  \sum_{j, j \neq i} | \sum_{p} [ K_{\Zset_{t-1} \Zset_{t-1}} - K_{\Xset_{t-1} \Xset_{t-1}} ]_{ip} 
	[ K_{\Zset_{t-1} \Zset_{t-1}} - K_{\Xset_{t-1} \Xset_{t-1}} ]_{pj} |\\ 
	\displaystyle  =  \sum_{j, j \neq i} | \sum_{p} ( k_{z_i z_p} - k_{x_i x_p})
	( k_{z_p z_j} - k_{x_p x_j}) |\\
	\displaystyle  =  \sum_{j, j \neq i} | \sum_{p, p \neq j,i} ( k_{z_i z_p} - k_{x_i x_p})
	( k_{z_p z_j} - k_{x_p x_j}) |\\
	\displaystyle  \le  \sum_{j, j \neq i}  \sum_{p, p \neq j,i} | k_{z_i z_p} - k_{x_i x_p} | \cdot 
	| k_{z_p z_j} - k_{x_p x_j} |\\
	\displaystyle  \le C^2 \sum_{j, j \neq i}  \sum_{p, p \neq j}  k_{x_i x_p}  \cdot 
	k_{x_p x_j} \\
	\displaystyle  =C^2   \sum_{p, p \neq j,i}  k_{x_i x_p}  \sum_{j, j \neq i, p} 
	k_{x_p x_j} \\
	\displaystyle  \le C^2   \sum_{p, p \neq j,i}  k_{x_i x_p} k_{x_p x_p} / \big( \sqrt{| 	\Xset_{t-1} |  -1} + 1 \big) \\
	\displaystyle  = C^2   \sigma_y^2 / \big( \sqrt{| 	\Xset_{t-1} |  -1} + 1 \big)  \sum_{p, p \neq j,  i}  k_{x_i x_p} \\
	\displaystyle  \le C^2   \sigma_y^2 / \big( \sqrt{| 	\Xset_{t-1} |  -1} + 1 \big) k_{x_i x_i}  / \big( \sqrt{| 	\Xset_{t-1} |  -1} + 1 \big) \\
	= \displaystyle C^2   \sigma_y^4 / \big( \sqrt{| 	\Xset_{t-1} |  -1} + 1 \big)^2
	\end{array}
	\end{equation*}
	where the third, the fifth and the last equalities follow from $k_{z_p z_p} = k_{x_p x_p} = \sigma_y^2$ for every $p$, the first inequality follows from triangle inequality, the second inequality follows from the statement of the lemma, the third and the last inequalities follow from the diagonal dominance property of $K_{\Xset_{t-1} \Xset_{t-1}}$ (Definition~\ref{definition-diag});
	and 
	\begin{equation*}
	\begin{array}{l}
	\displaystyle [( K_{\Zset_{t-1} \Zset_{t-1}} - K_{\Xset_{t-1} \Xset_{t-1}} )^2 ]_{ii} \\ 
	\displaystyle = \sum_{p}  [ K_{\Zset_{t-1} \Zset_{t-1}} - K_{\Xset_{t-1} \Xset_{t-1}} ]_{ip} 
	[ K_{\Zset_{t-1} \Zset_{t-1}} - K_{\Xset_{t-1} \Xset_{t-1}} ]_{pi} \\
	\displaystyle = \sum_{p}  [ K_{\Zset_{t-1} \Zset_{t-1}} - K_{\Xset_{t-1} \Xset_{t-1}} ]_{ip}^2 \\
	\displaystyle = \sum_{p}  ( k_{z_i z_p} - k_{x_i x_p})^2 \\
	\displaystyle = \sum_{p, p \neq i}  ( k_{z_i z_p} - k_{x_i x_p} )^2 \\
	\displaystyle \le C^2 \sum_{p, p \neq i}   k_{x_i x_p}^2 \\
	\displaystyle \le C^2 \big( \sum_{p, p \neq i}   k_{x_i x_p} \big)^2 \\
	\displaystyle \le C^2  k^2_{x_i x_i} 	/ \big( \sqrt{| \Xset_{t-1} |  -1} + 1 \big)^2 \\
	\displaystyle = C^2  \sigma_y^4 / \big( \sqrt{| \Xset_{t-1} |  -1} + 1 \big)^2    
	\end{array}
	\end{equation*} 
	where the second equality follows from the fact that $ K_{\Zset_{t-1} \Zset_{t-1}}$ and $K_{\Xset_{t-1} \Xset_{t-1}}$ are symmetric, the fourth and the last equalities follow from $k_{z_p z_p} = k_{x_p x_p} = \sigma_y^2$ for every $p$,  the first inequality follows from the statement of the lemma and  the last inequality follows from the diagonal dominance of $K_{\Xset_{t-1} \Xset_{t-1}}$  (Definition~\ref{definition-diag}).  			
\end{proof}

\begin{lemma}
	\label{lem:nghia}
	In the notations of Section~\ref{sec:main_proof}, if for all $t = 1, \ldots, T$ matrix $K_{\Xset_{t-1} \Xset_{t-1}}$ is  diagonally dominant (Definition~\ref{definition-diag}), then  for any unobserved original input $x \in \Xset$ at iteration $t$ 
	\begin{equation*}
	\lVert (K_{\Xset_{t-1} \Xset_{t-1}} + \sigma^2_n   I)^{-1} \rVert_2 \le 1 / (\sqrt{|\Xset_{t-1}|} \lVert K_{x \Xset_{t-1}}\rVert).
	\end{equation*}.   
\end{lemma}
\begin{proof}
	By applying Gershgorin circle theorem for $K_{{\Xset}_{t-1} {\Xset}_{t-1}}$:
	\begin{equation*}
	\begin{array}{l}
	\displaystyle \psi_{min}(K_{{\Xset}_{t-1} {\Xset}_{t-1}}) \\
	\displaystyle \ge \min_{x_i \in {\Xset}_{t-1}} \big( k_{x_i x_i} - R_{{\Xset}_{t-1}}(x_i) \big) \\
	\displaystyle = k_{xx} - \max_{x_i \in {\Xset}_{t-1}} R_{{\Xset}_{t-1}}(x_i) \\
	\displaystyle \ge (\sqrt{|\Xset_{t-1}|} + 1) \max_{x_i \in {\Xset}_{t-1} \cup \{x\} } R_{{\Xset}_{t-1}\cup \{x\} }(x_i) - \max_{x_i \in {\Xset}_{t-1}} R_{{\Xset}_{t-1}}(x_i)
	\end{array}
	\end{equation*}
	where  $\psi_{min}(\cdot)$ denotes the smallest eigenvalue of a matrix, $R_{{\Xset}_{t-1}}(x_i) \triangleq \sum_{x_j \in {\Xset}_{t-1} \setminus \{ x_i \}} k_{x_i x_j}$, the first equality follows from the fact that $k_{xx} = \sigma_y^2 = k_{x_i x_i}$ for all $x_i$ and $x$, and the second inequality holds because $K_{({\Xset}_{t-1} \cup \{ x\})({\Xset}_{t-1} \cup \{ x\})}$ is assumed to be diagonally dominant. On the other hand, since $x \notin \Xset_{t-1}$, $R_{{\Xset}_{t-1}\cup \{x\} }(x_i)  = R_{{\Xset}_{t-1}}(x_i) + k_{x_i x}$ for all $x_i \in \Xset_{t-1}$, which immediately implies $\max_{x_i \in {\Xset}_{t-1} \cup \{x\} } R_{{\Xset}_{t-1}\cup \{x\} }(x_i) \ge \max_{x_i \in {\Xset}_{t-1} } R_{{\Xset}_{t-1}\cup \{x\} }(x_i) \ge \max_{x_i \in {\Xset}_{t-1}} R_{{\Xset}_{t-1}}(x_i)$. Plugging this into above inequality,
	\begin{equation*}
	\begin{array}{l}
	\displaystyle \psi_{min}(K_{{\Xset}_{t-1} {\Xset}_{t-1}}) \\
	\displaystyle \ge (\sqrt{|\Xset_{t-1}|} + 1) \max_{x_i \in {\Xset}_{t-1} \cup \{x\} } R_{{\Xset}_{t-1}\cup \{x\} }(x_i) - \max_{x_i \in {\Xset}_{t-1}} R_{{\Xset}_{t-1}}(x_i) \\
	\displaystyle \ge \sqrt{|\Xset_{t-1}|} \max_{x_i \in {\Xset}_{t-1} \cup \{x\} } R_{{\Xset}_{t-1}\cup \{x\} }(x_i) \\
	\displaystyle \ge \sqrt{|\Xset_{t-1}|} R_{{\Xset}_{t-1}\cup \{x\} }(x).		
	\end{array}
	\end{equation*}
	Since $\lVert K_{x \Xset_{t-1}}\rVert = \sqrt{\sum_{x_i \in \Xset_{t-1}} k^2_{x_i x}} \le \sum_{x_i \in \Xset_{t-1}} k_{x_i x} = R_{{\Xset}_{t-1}\cup \{x\} }(x)$, it follows that $\psi_{min}(K_{{\Xset}_{t-1} {\Xset}_{t-1}}) \ge \sqrt{|\Xset_{t-1}|} \lVert K_{x \Xset_{t-1}}\rVert $. Finally,
	\begin{equation*}
	\begin{array}{l}
	\displaystyle \lVert (K_{{\Xset}_{t-1} {\Xset}_{t-1}} + \sigma^2_n  I)^{-1} \rVert_2 \\
	\displaystyle = 1 / (\psi_{min}(K_{{\Xset}_{t-1} {\Xset}_{t-1}})  + \sigma^2_n  I) \\
	\displaystyle \le 1 / (\psi_{min}(K_{{\Xset}_{t-1} {\Xset}_{t-1}})) \\
	\displaystyle \le 1 / \big( \sqrt{|\Xset_{t-1}|} \lVert K_{x \Xset_{t-1}}\rVert \big).		
	\end{array}
	\end{equation*}
\end{proof}

\begin{lemma}
	\label{lem:Kxx}
	In the notations of Section~\ref{sec:main_proof}, if  for all $t = 1, \ldots, T$ matrix $K_{\Xset_{t-1} \Xset_{t-1}}$ is  diagonally dominant (Definition~\ref{definition-diag}), then $\lVert K_{\Xset_{t-1}\Xset_{t-1}} \rVert_2 \le 2  \sigma_{y}^2$.
\end{lemma}

\begin{proof}
	Fix all $t = 1, \ldots, T$. By applying Gershgorin circle theorem to matrix $K_{{\Xset}_{t-1}{\Xset}_{t-1}}$,
	for some point $x_i \in \Xset_{t-1}$:
	\begin{equation*}
	\begin{array}{l}
	\displaystyle |\psi_{max}(K_{{\Xset}_{t-1}{\Xset}_{t-1}}) - k_{x_i x_i} | \\
	\displaystyle \le \sum_{x_j \in \Xset_{t-1} \setminus  x_i } k_{x_i x_j} \\
	\displaystyle \le  k_{x_i x_i} / \big( \sqrt{| \Xset_{t-1} |  -1} + 1 \big) \\
	\displaystyle = \sigma_{y}^2 / \big( \sqrt{| \Xset_{t-1} |  -1} + 1 \big) \\
	\end{array}
	\end{equation*}
	where  $\psi_{max}(\cdot)$ denotes the largest eigenvalue of a matrix, the second inequality is due to diagonal dominance property of matrix $K_{{\Xset}_{t-1}{\Xset}_{t-1}}$ and the equality is due to $k_{x_i x_i} = \sigma_{y}^2$ for every $x_i$. Since $K_{{\Xset}_{t-1}{\Xset}_{t-1}}$ is a symmetric, positive-semidefinite matrix, it follows that 
	\begin{equation*}
	\begin{array}{l}
	\displaystyle \lVert K_{{\Xset}_{t-1}{\Xset}_{t-1}} \rVert_2 \\
	\displaystyle  = \psi_{max}(K_{{\Xset}_{t-1}{\Xset}_{t-1}}) \\
	\le \sigma_{y}^2 / \big( \sqrt{| \Xset_{t-1} |  -1} + 1 \big) + k_{x_i x_i} \\
	\le \sigma_{y}^2 \big(1 + 1/ ( \sqrt{| \Xset_{t-1} |  -1} + 1 )\big)  \\
	\le  2 \sigma_{y}^2.
	\end{array}
	\end{equation*}. 
\end{proof}	
\begin{lemma}
	\label{lem:min_eigen}
	In the notations of Section~\ref{sec:main_proof}, for all $t = 1, \ldots, T$ and  any unobserved input $x \in \Xset$ at iteration $t$
	$\lVert K_{  x   {\Xset}_{t-1}} \rVert^2 
	\cdot \psi_{min}( ( K_{{\Xset}_{t-1}{\Xset}_{t-1}} + \sigma^2_n I)^{-1})
	\le K_{  x   {\Xset}_{t-1}}
	( K_{{\Xset}_{t-1}{\Xset}_{t-1}} + \sigma^2_n I)^{-1}
	K_{{\Xset}_{t-1}   x  }
	$	where  $\psi_{min}(\cdot)$ denotes the smallest eigenvalue of a matrix.
\end{lemma}
\begin{proof}
	Since $( K_{{\Xset}_{t-1}{\Xset}_{t-1}} + \sigma^2_n I)^{-1}$ is a symmetric, positive-definite matrix, there exists an orthonormal basis comprising the  eigenvectors $E \triangleq [e_1 \ldots e_{|\Xset_{t-1}|}]$  ($e_i^\top e_i =1$ and $e_i^\top e_j = 0$ for $i \neq j$) and their associated positive eigenvalues $\Psi^{-1}\triangleq \text{Diag}[\psi_{1}^{-1}, \ldots, \psi_{|\Xset_{t-1}|}^{-1}]$ such that $( K_{{\Xset}_{t-1}{\Xset}_{t-1}} + \sigma^2_n I)^{-1} = E \Psi^{-1} E^\top$ (i.e., spectral theorem). Denote $\{ p_i\}_{i=1}^{|\Xset_{t-1}|}$ as the set of coefficients when $K_{  {\Xset}_{t-1} x }$ is projected on $E$. Then 
	\begin{equation*}
	\begin{array}{l}
	\displaystyle K_{  x   {\Xset}_{t-1}}
	( K_{{\Xset}_{t-1}{\Xset}_{t-1}} + \sigma^2_n I)^{-1}
	K_{{\Xset}_{t-1}   x  } \\
	\displaystyle = \Bigg(  \sum_{i=1}^{|\Xset_{t-1}|} p_i e_i^\top \Bigg)
	( K_{{\Xset}_{t-1}{\Xset}_{t-1}} + \sigma^2_n I)^{-1} 
	\Bigg(  \sum_{i=1}^{|\Xset_{t-1}|} p_i e_i \Bigg) \\
	\displaystyle = \Bigg(  \sum_{i=1}^{|\Xset_{t-1}|} p_i e_i^\top \Bigg)
	\Bigg(  \sum_{i=1}^{|\Xset_{t-1}|} p_i ( K_{{\Xset}_{t-1}{\Xset}_{t-1}} + \sigma^2_n I)^{-1}  e_i  \Bigg) \\
	\displaystyle = \Bigg(  \sum_{i=1}^{|\Xset_{t-1}|} p_i e_i^\top \Bigg)
	\Bigg(  \sum_{i=1}^{|\Xset_{t-1}|} p_i \psi_{i}^{-1}  e_i  \Bigg) \\ 
	\displaystyle =  \sum_{i=1}^{|\Xset_{t-1}|} p^2_i \psi_{i}^{-1} \\ 
	\displaystyle \ge  \psi_{min}( ( K_{{\Xset}_{t-1}{\Xset}_{t-1}} + \sigma^2_n I)^{-1})  \,  \sum_{i=1}^{|\Xset_{t-1}|} p^2_i \\ 
	\displaystyle =  \psi_{min}( ( K_{{\Xset}_{t-1}{\Xset}_{t-1}} + \sigma^2_n I)^{-1})  \,  	\lVert K_{  x   {\Xset}_{t-1}} \rVert^2.
	\end{array}
	\end{equation*}
\end{proof}
\end{document}
